\documentclass{article}

\usepackage{times}
\usepackage{graphicx} 

\usepackage{natbib}

\usepackage[nohyperref,accepted]{icml2017}

\usepackage[utf8]{inputenc} 
\usepackage[T1]{fontenc}    
\usepackage{url}            
\usepackage{booktabs}       
\usepackage{amsfonts}       
\usepackage{nicefrac}       
\usepackage{microtype}      

\usepackage{amsmath,amsthm,amsfonts,amssymb}

\usepackage{graphicx}
\usepackage{subcaption}
\usepackage{tikz}
\usetikzlibrary{automata,arrows,positioning,calc}
\usepackage{array}

\usepackage{algorithm}
\usepackage[noend]{algpseudocode}

\newcommand{\RR}{\mathbb{R}} 

\newcommand{\xdim}{d}

\newcommand{\nstates}{n}
\usepackage{mathtools}
\newcommand{\defeq}{\vcentcolon=}

\newcommand{\Actions}{\mathcal{A}}
\newcommand{\States}{\mathcal{S}}
\newcommand{\Pfcn}{\mathrm{Pr}}
\newcommand{\Rfcn}{r}
\newcommand{\Ppi}{\mathbf{P}_\pi}
\newcommand{\Pl}{{\Ppi^{\lambda}}}
\newcommand{\Plq}{{\Ppi^{\lambda,q}}}
\newcommand{\Pgam}{\mathbf{P}_{\pi,\gamma}}

\newcommand{\dpi}{\mathbf{d}_{\pi}}
\newcommand{\rpi}{\rvec_\pi}
\newcommand{\rl}{\rvec_\pi^{\lambda}}

\newcommand{\vpi}{\mathbf{v}_{\pi}}

\newcommand{\Amat}{\mathbf{A}}

\newcommand{\Dmat}{\mathbf{D}}

\newcommand{\Mmat}{\mathbf{M}}
\newcommand{\Pmat}{\mathbf{P}}
\newcommand{\Qmat}{\mathbf{Q}}

\newcommand{\Xmat}{\mathbf{X}}

\newcommand{\Mmatpi}{\mathbf{M}}

\renewcommand{\vec}[1]{\mathbf{\boldsymbol{#1}}}
\newcommand{\zerovec}{\vec{0}}
\newcommand{\onevec}{\vec{1}}

\newcommand{\bvec}{\mathbf{b}}
\newcommand{\dvec}{\mathbf{d}}
\newcommand{\evec}{\mathbf{e}}

\newcommand{\ivec}{\mathbf{i}}

\newcommand{\mvec}{\mathbf{m}}

\newcommand{\qvec}{\mathbf{q}}
\newcommand{\rvec}{\mathbf{r}}

\newcommand{\vvec}{\mathbf{v}}
\newcommand{\wvec}{\mathbf{w}}
\newcommand{\xvec}{\mathbf{x}}
\newcommand{\zvec}{\mathbf{z}}

\newcommand{\vvlambda}{\vvec_{\pi,\lambda}}
\newcommand{\vvpi}{\vvec_\pi}

\newcommand{\maxlam}{r}

\newcommand{\xcirc}[1]{\vcenter{\hbox{$#1\circ$}}}
\newcommand{\ccirc}{\mathbin{\mathchoice
  {\xcirc\scriptstyle}
  {\xcirc\scriptstyle}
  {\xcirc\scriptscriptstyle}
 {\xcirc\scriptscriptstyle}
}}
\newcommand{\hada}{\ccirc}

\newcommand{\fakes}{f}
\newcommand{\fakeset}{\mathcal{F}}
\newcommand{\fakeP}{\bar{\text{P}}\text{r}}
\newcommand{\fakeR}{\bar{\Rfcn}}
\newcommand{\fakeppi}{\bar{\Pmat}_{\pi}}
\newcommand{\fakerpi}{\bar{\rvec}_{\pi}}
\newcommand{\fakegam}{\bar{\gamma}_s}
\newcommand{\fakeS}{\bar{\States}}
\newcommand{\fakedpi}{\bar{\dvec}_{\pi}}
\newcommand{\fakepi}{\bar{\pi}}
\newcommand{\fakevpi}{\bar{\vvec}_{\pi}}

\newcommand{\Dpi}{\mathbf{D}_\pi}

\newcommand{\Tl}{\mathrm{T}^{(\lambda)}}

\newcommand{\Tlq}{T^{(\lambda,q)}}

\newcommand{\Dmu}{\mathbf{D}_{\mu}}

\newcommand{\dmu}{\mathbf{d}_\mu}

\newcommand{\fake}{hypothetical}
 \newcommand{\sbound}[1]{s_{#1}}
 \newcommand{\Proj}{\Pi}
  \newcommand{\Fallv}{ \mathcal{V}}
  \newcommand{\Fsubspace}{ \mathcal{F}_v}

\newcommand{\rlq}{\rvec_q}

\newcommand{\Pgamq}{\mathbf{P}_{\pi,\gamma,q}}
\newcommand{\Pgamlamq}{\mathbf{P}_{\pi,\gamma,\lambda,q}}

\newcommand{\ivecq}{\ivec_q}
\newcommand{\dmuq}{\mathbf{d}_{\mu,q}}

\newcommand{\Xmatq}{\Xmat_q}
\newcommand{\Mmatq}{\Mmat_q}
\newcommand{\Amatq}{\Amat_q}
\newcommand{\bvecq}{\bvec_q}
\newcommand{\piset}{\mathcal{P}}

\newcommand{\eye}{\mathbf{I}}

\newcommand{\inv}{{\raisebox{.2ex}{$\scriptscriptstyle-1$}}}
\newcommand{\diag}{\mathop{\mathrm{diag}}}
\newcommand{\spnorm}[1]{\left\|  #1 \right\|_{sp}}
\newcommand*{\argmin}{\mathop{\mathrm{argmin}}}
\newcommand{\palignbox}{\par \vspace{-0.8cm}}
\newcommand*{\argmax}{\mathop{\mathrm{argmax}}}

\newtheorem{proposition}{Proposition}
\newtheorem{lemma}{Lemma}
\newtheorem{theorem}{Theorem}
\newtheorem{corollary}{Corollary}

\icmltitlerunning{Unifying Task Specification in Reinforcement Learning}

\begin{document}
 
\twocolumn[
\icmltitle{Unifying Task Specification in Reinforcement Learning}

\begin{icmlauthorlist}
\icmlauthor{Martha White}{iu}
\end{icmlauthorlist}

\icmlaffiliation{iu}{Department of Computer Science, Indiana University}
\icmlcorrespondingauthor{Martha White}{martha@indiana.edu}

\icmlkeywords{reinforcement learning,markov decision process,value functions}

\vskip 0.2in
]
\printAffiliationsAndNotice{}  

\begin{abstract}
Reinforcement learning tasks are typically specified as Markov decision processes. This formalism has been highly successful, though
specifications often couple the dynamics of the environment and the learning
objective. This lack of modularity can complicate generalization of the task specification, as well as obfuscate connections between different task settings,
such as episodic and continuing. 
In this work, we introduce the RL task formalism, that provides a unification through simple constructs including a generalization to transition-based discounting.
Through a series of examples, we demonstrate the generality and utility 
of this formalism. Finally, we extend standard learning constructs, including Bellman operators, and extend some seminal theoretical results, including approximation errors bounds. 
Overall, we provide 
a well-understood and sound formalism on which to build
theoretical results and simplify algorithm use and development.  
\end{abstract}

\section{Introduction}

Reinforcement learning is a formalism for trial-and-error interaction
between an agent and an unknown environment.
This interaction is typically specified by a Markov decision process (MDP),
which contains a transition model, reward model,
and potentially discount parameters $\gamma$ specifying a discount on
the sum of future values in the return.
Domains 
are typically separated into
two cases: episodic problems (finite horizon) and continuing problems (infinite horizon).
In episodic problems, the agent reaches some terminal state,
and is teleported back to a start state.
In continuing problems, the agent interaction is continual, with a discount
to ensure a finite total reward (e.g., constant $\gamma < 1$). 

This formalism has a long and successful tradition, but is limited in the problems that can be specified.
Progressively there have been additions 
to specify a broader range of objectives, including options \citep{sutton1999between}, state-based discounting \citep{sutton1995td,sutton2011horde} and interest functions \citep{maei2010gq,sutton2016anemphatic}. 
These generalizations have particularly been driven by off-policy learning
and the introduction of general value functions for Horde \citep{sutton2011horde,white2015thesis},
where predictive knowledge can be encoded as more complex prediction
and control tasks. 
Generalizations to problem specifications provide exciting learning opportunities, but can
also reduce clarity 
and complicate algorithm development and theory. 
For example, options and general value functions have significant overlap, but because of
different terminology and formalization, the connections are not transparent.
Another example is the classic divide between episodic and continuing problems,
which typically require different convergence proofs \citep{bertsekas1996neuro,tsitsiklis1997ananalysis,sutton2009fast}
and different algorithm specifications.

In this work, we propose a formalism for reinforcement learning task specification that unifies many of these generalizations. 
The focus of the formalism is to separate the specification of the dynamics
of the environment and the specification of the objective within that environment.
Though natural, this represents a significant change in the way tasks are currently
specified in reinforcement learning and has important ramifications for
simplifying implementation, algorithm development and theory.
The paper consists of two main contributions.
First, we demonstrate the utility of this formalism by showing unification
of previous tasks specified in reinforcement learning, including options, general value functions and episodic and continuing,
and further providing case studies of utility.
We demonstrate how to specify episodic and continuing tasks with only modifications
to the discount function, 
without the addition of states and modifications to the underlying Markov decision process. This
enables a unification that 
significantly simplifies implementation and easily generalizes theory to cover both settings.
Second, 
we prove novel contraction bounds on the Bellman operator
for these generalized RL tasks, and show that previous bounds for
both episodic and continuing tasks are subsumed by this more general result.  
Overall, our goal is to provide an RL task formalism that requires minimal
modifications to previous task specification, with significant gains in simplicity and unification across common settings.  

\section{Generalized problem formulation}

We assume the agent interacts with an environment
formalized by a Markov decision process (MDP): $(\States, \Actions, \Pfcn)$ 
where 
$\States$ is the set of states, $\nstates = | \States |$;
$\Actions$ is the set of actions; and
$\Pfcn: \States \times \Actions \times \States \rightarrow [0,1]$ is the transition probability function
where
$\Pfcn(s,a,s')$ is the probability of transitioning from state $s$ into state $s'$ when taking action $a$.
A \textbf{reinforcement learning task} (RL task) is specified on top of these transition dynamics, as the tuple
$(\piset, \Rfcn, \gamma, \ivec)$ where
\begin{enumerate}
\item $\piset$ is a set of policies $\pi: \States \times \Actions \rightarrow [0,1]$;
\item the reward function $\Rfcn: \States \times \Actions \times \States \rightarrow \RR$ specifies reward received
from $(s,a,s')$;
\item $\gamma: \States \times \Actions \times \States \rightarrow [0,1]$ is a \textbf{transition-based
discount function}\footnote{We describe a further probabilistic generalization
in Appendix \ref{app_probabilistic}; much of the treatment
remains the same, but the notation becomes cumbersome and the utility obfuscated.};
\item $\ivec : \States \rightarrow [0, \infty)$ is an interest function that specifies the user defined interest in a state.
\end{enumerate}
Each task could have different reward functions within the same environment. 
For example, in a navigation task within an office, one agent could have the goal to navigate to
the kitchen and the other the conference room.
For a reinforcement learning task, whether prediction or control, 
a set or class of policies is typically considered.
For prediction (policy evaluation), we often select one policy and evaluate its long-term discounted reward.
For control, where a policy is learned, the set of policies may consist of all policies parameterized
by weights that specify the action-value from states, with the goal to find the weights
that yield the optimal policy.
For either prediction or control in an RL task, we often evaluate the return of a policy: the cumulative discounted
reward obtained from following that policy
\begin{align*}
G_t = \sum_{i=0}^\infty \left(\prod_{j=0}^{i-1} \gamma(s_{t+j}, a_{t+j}, s_{t+1+j}) \right) R_{t+1+i}
\end{align*}
where $\prod_{j=0}^{-1} \gamma(s_{t+j}, a_{t+j}, s_{t+1+j}) \defeq 1$.
Note that this subsumes the setting with a constant discount $\gamma_c \in [0,1)$, 
by using $\gamma(s,a,s') = \gamma_c$ for every $s, a, s'$
and so giving $\prod_{j=0}^{i-1} \gamma(s_{t+j}, a_{t+j}, s_{t+1+j}) = \gamma_c^{i}$ for $i > 0$
and $\gamma_c^0 = 1$ for $i = 0$.
As another example, the end of the episode, $\gamma(s,a,s') = 0$, making the product
of these discounts zero and so terminating the recursion.
We further explain how transition-based discount enables specification of episodic tasks
and discuss the utility of the generalization to transition-based discounting throughout this paper.
Finally, the interest function $\ivec$ specifies the degree of importance of each state for the task.
For example, if an agent is only interested in learning an optimal policy for a subset of the environment,
the interest function could be set to one for those states and to zero otherwise. 

We first explain the specification and use of such tasks,
and then define 
a generalized Bellman operator and resulting algorithmic extensions and approximation bounds.

\subsection{Unifying episodic and continuing specification}

\newcommand{\raction}{\textit{right}}
\newcommand{\laction}{\textit{left}}
\newcommand{\sone}{s_1}
\newcommand{\stwo}{s_2}
\newcommand{\sthree}{s_3}
\newcommand{\sfour}{s_4}

The RL task specification enables episodic and continuing problems to
be easily encoded with only modification to the transition-based discount.
Previous approaches, including the absorbing state formulation \citep{sutton1998reinforcement}
and state-based discounting \citep{sutton1995td,maei2010gq,sutton2011horde}\citep[Section 2.1.1]{hado2011thesis}, require special cases or modifications to the set of states and underlying MDP,
coupling task specification and the dynamics of the environment. 

We demonstrate how transition-based discounting seamlessly enables episodic
or continuing tasks to be specified in an MDP via a simple chain world.
Consider the chain world with three states $\sone$, $\stwo$ and $\sthree$ in Figure \ref{fig:unification}.
The start state is $\sone$ and the two actions are \raction\ and \laction.
The reward is -1 per step, with termination occurring when taking
action \raction\ from state $\sthree$, which causes a transition back to state $\sone$.
The discount is 1 for each step, unless specified otherwise. 
The interest is set to 1 in all states, which is the typical case, meaning 
performance from each state is equally important.

Figure \ref{fig_absorbing} depicts the classical approach to specifying episodic problems
using an absorbing state, drawn as a square. The agent reaches the goal---transitioning
right from state $\sthree$---then forever stays in the absorbing state, 
receiving a reward of zero.
This encapsulates the definition of the return, but does not
allow the agent to start another episode. In practice, when this absorbing
state is reached, the agent is ``teleported" to a start state to start another episode.
This episodic interaction can instead be represented the same way as a continuing
problem, by specifying a transition-based discount $\gamma(\sthree,\raction, \sone) = 0$.
This defines the same return, but now the agent simply transitions normally
to a start state, and no \fake\ states are added.
 
To further understand the equivalence,
consider the updates made by TD (see equation \eqref{eq_td}). 
Assume linear function approximation with feature function $\xvec: \States \rightarrow \RR^\xdim$,
with weights $\wvec \in \RR^\xdim$. 
When the agent takes action \raction\
from $\sthree$, the agent transitions from $\sthree$ to $\sone$ with probability one
and so $\gamma_{t+1} = \gamma(\sthree, \raction,\sone) = 0$.
This correctly gives
\begin{align*}
\delta_t &= r_{t+1} + \gamma_{t+1} \xvec(\sone)^\top \wvec - \xvec(\sthree)^\top \wvec 
= r_{t+1} - \xvec(\sthree)^\top \wvec
\end{align*}
and correctly clears the eligibility trace for the next step
\begin{align*}
\evec_{t+1} &= \lambda_{t+1} \gamma_{t+1} \evec_t + \xvec(\sone)
= \xvec(\sone)
.
\end{align*}
The stationary distribution is also clearly equal
to the original episodic task, since the absorbing state
is not used in the computation of the stationary distribution.

Another strategy is to still introduce \fake\ states, but use
state-based $\gamma$, as discussed in Figure \ref{fig_unification_state}. Unlike
absorbing states, the agent does not stay indefinitely in
the \fake\ state.
When the agent goes \raction\ from $\sthree$, it transitions to \fake\ state $\sfour$,
and then transition deterministically to the start state $\sone$, with $\gamma_s(\sfour) = 0$. 
As before, we get the correct update, because $\gamma_{t+1} = \gamma_s(\sfour) = 0$.
Because the stationary distribution has some non-zero probability in
the \fake\ state $\sfour$, 
we must set $\xvec(\sfour) = \xvec(\sone)$ (or $\xvec(\sfour) = \zerovec$).
Otherwise,
the value of the \fake\ state will be learned, wasting
function approximation resources and potentially
modifying the approximation quality of the value in other states.
We could have tried
state-based discounting without adding an additional state $\sfour$.
However, this leads to incorrect value estimates, as depicted
in Figure \ref{fig:unificationwrong}; the relationship between transition-based
and state-based is further discussed in Appendix \ref{sec_equivalence}.
Overall, to keep the specification of the RL task and the MDP separate,
transition-based discounting is necessary to enable the unified
specification of episodic and continuing tasks. 

\newcommand{\figwidth}{.495\textwidth}
\newcommand{\tikzsize}{190pt}
\newcommand{\tikzscale}{0.9}
\newcommand{\labelpaa}{}
\newcommand{\labelpab}{}
\newcommand{\labelpbc}{}
\newcommand{\labelpba}{}
\newcommand{\labelpcd}{}
\newcommand{\labelpcb}{}

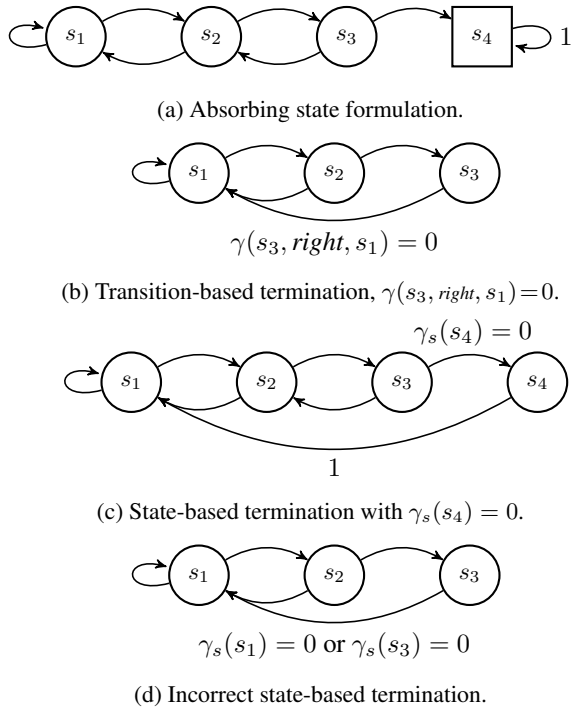
\begin{figure}[h!]
\vspace{-0.1cm}
\centering
\begin{subfigure}{\figwidth}
		{\begin{tikzpicture}[->, >=stealth', auto, semithick, node distance=2cm]
\tikzstyle{every state}=[fill=white,draw=black,thick,text=black,scale=\tikzscale]
\node[state]    (A)                     {$\sone$};
\node[state]    (B)[right of=A]   {$\stwo$};
\node[state]    (C)[right of=B]   {$\sthree$};
\node[state,rectangle]    (D)[right of=C]   {$\sfour$};
\path
(A) edge[loop left]     node{}         (A)
    edge[bend left]     node{}    (B)
(B) edge[bend left = 30]       node{}          (C)
      edge[bend right = -30]     node{}      (A)
(C) edge[bend left = 30]       node{}          (D)
      edge[bend right = -30]     node{}      (B)
(D) edge[loop right]    node{$1$}     (D);
\end{tikzpicture}}
		\caption{Absorbing state formulation.
		}\label{fig_absorbing}
\end{subfigure}
\begin{subfigure}[r]{\figwidth}
		\centerline{\begin{tikzpicture}[->, >=stealth', auto, semithick, node distance=2cm]
\tikzstyle{every state}=[fill=white,draw=black,thick,text=black,scale=\tikzscale]
\node[state]    (A)                     {$\sone$};
\node[state]    (B)[right of=A]   {$\stwo$};
\node[state]    (C)[right of=B]   {$\sthree$};
\path
(A) edge[loop left]     node{}         (A)
    edge[bend left]     node{}     (B)
(B) edge[bend left = 30]       node{}           (C)
      edge[bend right = -30]     node{}      (A)
(C) edge[bend right = -30]       node{$\gamma(\sthree,\raction,\sone) = 0$}           (A);
\end{tikzpicture}}
		\caption{Transition-based termination, $\scriptsize\gamma(\sthree,\raction,\sone) \!=\! 0$.}
\end{subfigure}\\
\begin{subfigure}{\figwidth}
		\centerline{\begin{tikzpicture}[->, >=stealth', auto, semithick, node distance=2cm]
\tikzstyle{every state}=[fill=white,draw=black,thick,text=black,scale=\tikzscale]
\node[state]    (A)                     {$\sone$};
\node[state]    (B)[right of=A]   {$\stwo$};
\node[state]    (C)[right of=B]   {$\sthree$};
\node[state]    (D)[right of=C]   {$\sfour$};
\path
(A) edge[loop left]     node{\labelpaa}         (A)
    edge[bend left]     node{\labelpab}     (B)
(B) edge[bend left = 30]       node{\labelpbc}           (C)
      edge[bend right = -30]     node{\labelpba}      (A)
(C) edge[bend left = 30]       node{\labelpcd \ $\gamma_s(\sfour) = 0$}           (D)
      edge[bend right = -30]     node{\labelpcb}      (B)
(D) edge[bend right = -30]     node{$1$}      (A);
\end{tikzpicture}}
		\caption{State-based termination with $\gamma_s(\sfour) = 0$.}\label{fig_unification_state}		
\end{subfigure}
\begin{subfigure}{\figwidth}
		\centerline{\begin{tikzpicture}[->, >=stealth', auto, semithick, node distance=2cm]
\tikzstyle{every state}=[fill=white,draw=black,thick,text=black,scale=\tikzscale]
\node[state]    (A)                     {$\sone$};
\node[state]    (B)[right of=A]   {$\stwo$};
\node[state]    (C)[right of=B]   {$\sthree$};
\path
(A) edge[loop left]     node{}         (A)
    edge[bend left]     node{}     (B)
(B) edge[bend left = 30]       node{}           (C)
      edge[bend right = -30]     node{}      (A)
(C) edge[bend right = -30]       node{$\gamma_s(\sone) = 0$ or $\gamma_s(\sthree) = 0$}           (A);
\end{tikzpicture}
%
}
		\caption{Incorrect state-based termination.}\label{fig:unificationwrong}
\end{subfigure}
\vspace{-0.3cm}
	\caption{ Three different ways to represent episodic problems as continuing problems.
	For (c), the state-based discount cannot represent the episodic chain problem without adding states.
To see why, consider the two cases for representing termination:
$\gamma_s(\sone) = 0$ or $\gamma_s(\sthree) = 0$.
For simplicity, assume that $\pi(s,\raction) = 0.75$ for all states $s \in \{\sone,\stwo,\sthree\}$
and transitions are deterministic.
If $\gamma_s(\sthree) = 0$, then the value for taking action \raction\ from $\stwo$ is
$r(\stwo,\raction,\sthree) + \gamma_s(\sthree) \vvpi(\sthree) = -1$
and the value for taking action \raction\ from $\sthree$ is 
$r(\sthree,\raction,\sone) + \gamma_s(\sone) \vvpi(\sone) \neq -1$,
which are both incorrect.
If $\gamma_s(\sone) = 0$, then the value of 
taking action \raction\ from $\sthree$ is $-1 + \gamma_s(\sone) \vvpi(\sone) = -1$,
which is correct. However,
the value of taking action \laction\ from $\stwo$
is $-1 + \gamma_s(\sone) \vvpi(\sone) = -1$,
which is incorrect.	}
	\label{fig:unification}
\end{figure}

\subsection{Options as RL tasks}\label{sec_options}

The options framework \citep{sutton1999between} generically covers a wide range of settings,
with discussion about macro-actions, option models, interrupting options and intra-option value learning.
These concepts at the time merited their own language, but with recent generalizations
can be more conveniently cast as RL subtasks.
%

\begin{proposition}
An option, defined as the tuple \citep[Section 2]{sutton1999between} $(\pi, \beta, \mathcal{I})$ 
with policy $\pi: \States \times \Actions \rightarrow [0,1]$,
termination function $\beta: \States \rightarrow [0,1]$
and an initiation set $\mathcal{I} \subset \States$ from which the option can be run,  
can be equivalently cast as an RL task.
\end{proposition}
This proof is mainly definitional, but we state it as an explicit proposition for clarity. 
The discount function $\gamma(s,a,s') = 1-\beta(s')$ for all $s,a,s'$ specifies termination.
The interest function, $\ivec(s) = 1$ if $s \in \mathcal{I}$ and $\ivec(s) = 0$ otherwise,
focuses learning resources on the states of interest. If a value function for the policy is queried,
it would only make sense to query it from these states of interest. If
the policy for this option is optimized for this interest function, the policy should
only be run starting from $s \in \mathcal{I}$, as elsewhere will be poorly learned.
The rewards for the RL task correspond to the rewards associated with the MDP.

RL tasks generalize options, by generalizing
termination conditions to transition-based discounting and by providing degrees of interest
rather than binary interest. 
Further, the policies associated with RL subtasks can be used as macro-actions,
to specify a semi-Markov decision process \citep[Theorem 1]{sutton1999between}.


\subsection{General value functions}

In a similar spirit of abstraction as options, general value functions were introduced for 
single predictive or goal-oriented questions about the world \citep{sutton2011horde}.
The idea is to encode predictive knowledge in the form of value function predictions:
with a collection or horde of prediction demons, this constitutes knowledge \citep{sutton2011horde,modayil2014multi,white2015thesis}. 
The work on Horde \citep{sutton2011horde} and nexting \citep{modayil2014multi} provide numerous examples of the utility of the types of questions that
can be specified by general value functions, and so by RL tasks,
because general value functions can naturally can be specified as an RL task.

%
The generalization to RL tasks provide additional benefits for predictive knowledge. 
The separation into underlying MDP dynamics and task specification
is particularly useful in off-policy learning, with the Horde formalism, where many demons (value functions)
are learned off-policy. These demons share the underlying dynamics, and even feature representation, but have separate prediction and control tasks;
keeping these separate from the MDP is key for avoiding complications (see Appendix \ref{sec_advantages}).
%
 Transition-based discounts, over state-based discounts, additionally enable the prediction of a \textit{change},
caused by transitioning between states.
Consider the taxi domain, described more fully in Section \ref{sec_taxi},
where the agent's goal is to pick up and drop off passengers in a grid world with walls.
The taxi agent may wish to predict the probability of hitting
a wall, when following a given policy. This can be encoded by setting $\gamma(s,a,s) = 0$
if a movement action causes the agent to remain in the same state, which occurs when trying to
move through a wall. 
In addition to episodic problems and hard termination,
transition-based questions also enable soft termination for transitions.
Hard termination uses $\gamma(s,a,s') = 0$ and soft termination
$\gamma(s,a,s') = \epsilon$ for some small positive value $\epsilon$. 
Soft terminations can be useful for incorporating some
of the value of a policy right after the soft termination. If two policies are equivalent
up to a transition, but have very different returns after the transition, a soft termination
will reflect that difference. We empirically demonstrate the utility of soft termination in the next section.


\begin{figure*}[t]
\centering 
\vspace{-0.25cm}
\begin{subfigure}[t]{0.7\textwidth}
\hspace{-0.6cm}
 \includegraphics[width=1.0\textwidth]{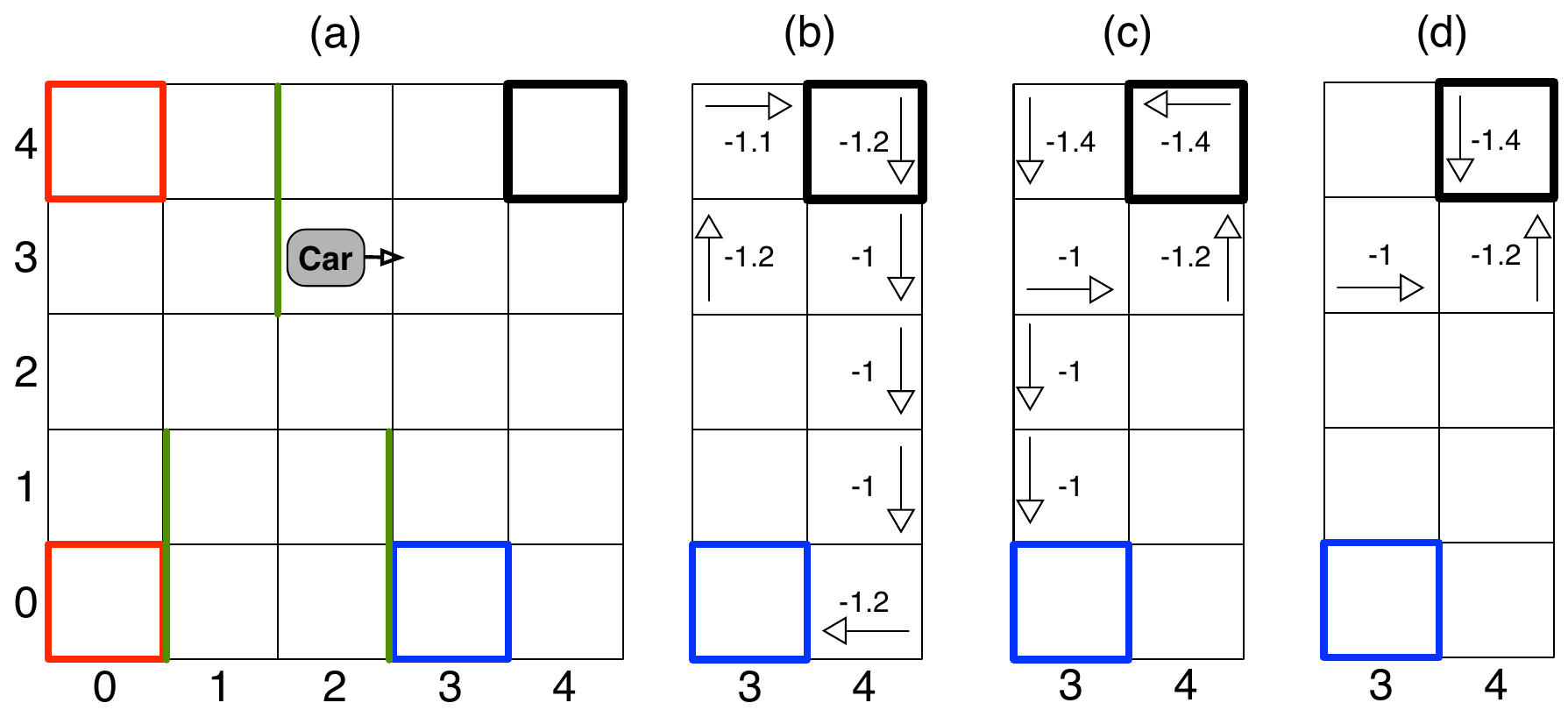}
\end{subfigure}
 \hspace{-0.6cm}
\begin{subfigure}[t]{0.28\textwidth}
\vspace{-5.4cm}
\begin{minipage}{1.0\textwidth}
\centering
\textbf{(e)}
\vspace{0.25cm}
\end{minipage}
{\footnotesize
\setlength\tabcolsep{1pt}
\begin{sc}
\begin{tabular}{|c|p{18mm}|p{16mm}|}
\hline
& {\scriptsize Total Pickup and Dropoff} & {\scriptsize Added Cost for Turns} \\
\hline
  {\scriptsize Trans-Soft}  & $7.74 \pm 0.03$ & $5.54 \pm 0.01$\\
  {\scriptsize  Trans-Hard} & $7.73 \pm 0.03$ & $5.83 \pm 0.01$\\
  {\scriptsize State-based} & $0.00 \pm 0.00$ & $18.8 \pm 0.02$\\
  \hline
  {\scriptsize $\gamma_c = 0.1$} & $0.00 \pm 0.00$ & $2.48 \pm 0.01$\\
  {\scriptsize $\gamma_c = 0.3$} & $0.02 \pm 0.01$ & $2.49 \pm 0.01$\\ 
  {\scriptsize $\gamma_c = 0.5$} & $0.04 \pm 0.01$ & $2.51 \pm 0.01$\\ 
  {\scriptsize $\gamma_c = 0.6$} & $0.03 \pm 0.01$ & $2.49 \pm 0.01$\\ 
 {\scriptsize $\gamma_c = 0.7$} & $7.12 \pm 0.03$ & $4.52 \pm 0.01$\\ 
  {\scriptsize $\gamma_c = 0.8$} & $7.34 \pm 0.03$ & $4.62 \pm 0.01$\\ 
  {\scriptsize $\gamma_c = 0.9$} & $3.52 \pm 0.06$ & $4.57 \pm 0.02$\\ 
  {\scriptsize $\gamma_c = 0.99$} & $0.01 \pm 0.01$ & $2.45 \pm 0.01$\\
\hline
\end{tabular}
\end{sc}
}
\end{subfigure}
 \caption{\small \textbf{(a)} The taxi domain, where the pickup/drop-off platforms are at (0,0), (0,4), (3,0) and (4,4). 
The Passenger P is at the source platform (4,4), outlined in black. The Car starts in (2,3), with orientation $E$ as indicated the arrow, needs to bring the passenger to destination D platform at (3,0), outlined in blue.
 In \textbf{(b) - (d)}, there are simulated trajectories for policies learned using hard and soft termination. \\
 \textbf{(b)} The optimal strategy, with $\gamma$(Car in source, Pickup, P in Car) = 0.1
 and a discount 0.99 elsewhere. The sequence of
 taxi locations are
 $(3,3)$, $(3,4)$, $(4,4)$, $(4,4)$ with Pickup action, $(4,3)$, $(4,2)$, $(4,1)$, $(4,0)$, $(3,0)$. Successful pickup and drop-off with total reward $-7.7$.\\
 \textbf{(c)} For $\gamma$(Car in source, Pickup, P in Car) = 0, 
 the agent does not learn the optimal strategy.
 The agent minimizes orientation cost to the subgoal, not accounting
 for orientation after picking up the passenger. 
 Consequently, it takes more left turns after pickup,
 resulting in more total negative reward. 
 The sequence of
 locations are
 $(3,3)$, $(4,3)$, $(4,4)$, $(4,4)$ with Pickup action, $(3,4)$, $(3,3)$, $(3,2)$, $(3,1)$, $(3,0)$.
		Successful pickup and drop-off with total reward $-8$.\\
 \textbf{(d)} For state-based $\gamma$(Car in source and P in Car) = 0, 
 the agent remains around the source and does not complete a successful drop-off.
    The sequence of
 locations are
  $(3,3)$, $(4,3)$, $(4,4)$, $(4,4)$ with Pickup action, $(4,3)$, $(4,4)$, $(4,3)$....
  The agent enters the source and pickups up the passenger.
  When it leaves to location (4,3), its value function
  indicates better value going to (4,4) because
  the negative return will again be cutoff by $\gamma$(Car in source and P in Car) = 0,
  even without actually performing a pickup. Since the cost to get to the destination
  is higher than the $-2.6$ return received from going back to $(4,4)$, the agent
  stays around $(4,4)$ indefinitely.\\
 \textbf{(e)} Number of successful passenger pickup and dropoff, as well as additional cost incurred from turns,
 over 100 steps, with 5000 runs, reported
 for a range of constant $\gamma_c$ and the policies in Figure \ref{fig_taxi}.
  Due to numerical imprecision, several constant discounts do not get close enough to the passenger to pickup or drop-off. 
 The state-based approach, that does not add additional states for termination, oscillates
 after picking up the passenger, and so constantly gets negative reward.
 }\label{fig_taxi}
 \vspace{-0.3cm}
\end{figure*}

\section{Demonstration in the taxi domain}\label{sec_taxi}

\newcommand{\plat}{-----}
\newcommand{\offset}{-3}
\newcommand{\passenger}{\node at (\offset+4.8,4.8) {\small\textcolor{gray}{P}};}
\newcommand{\destination}{\node at (\offset+3.5,0.5) {\textcolor{gray}{D}};}

\newcommand{\gsize}{g}
\newcommand{\psize}{l}

To better ground this generalized formalism and provide some intuition,
we provide a demonstration of RL task specification.
We explore different transition-based discounts in
the taxi domain \citep{dietterich2000hierarchical,diuk2008anobject}.
The goal of the agent is to take a passenger from a source platform 
to a destination platform, depicted in Figure \ref{fig_taxi}.
The agent receives a reward of -1 on each step,
except for successful pickup and drop-off, giving reward 0.
We modify the domain to include the orientation of the taxi,
with additional cost for not
continuing in the current orientation. 
This encodes that turning right, left or going backwards are more costly than going
forwards, with additional negative rewards of -0.05, -0.1 and -0.2 respectively. 
This additional cost is further multiplied by a factor of 2 when there is a passenger in the vehicle. 
For grid size $\gsize$ and the number of pickup/dropoff locations $\psize$, 
the full state information
is a 5-tuple: 
($x$ position of taxi $\in \{1,\ldots,\gsize\}$,
$y$ position of taxi $\in \{1,\ldots,\gsize\}$, 
location of passenger $\in \{1,\ldots,\psize+1\}$, 
location of destination $\in \{1,\ldots,\psize\}$, 
orientation of car $\in \{N,E,S,W\}$
).
The location of the passenger can be in one of the pickup/drop-off locations,
or in the taxi. Optimal policies and value functions are computed iteratively, with an extensive number of iterations.

Figure \ref{fig_taxi} illustrates three policies for one part of the taxi domain,
obtained with three different discount functions. The optimal 
policy is learned using a soft-termination, which takes into consideration
the importance of approaching the passenger location with the right orientation
to minimize turns after picking up the passenger. A suboptimal policy is in fact
learned with hard termination, 
as the policy prefers to greedily
minimize turns to get to the passenger. 
For further details, refer to the caption in Figure \ref{fig_taxi}.

We also compare to a 
 constant $\gamma$,
which corresponds to an average reward goal, as demonstrated in Equation \eqref{eq_average}. 
The table in Figure \ref{fig_taxi}(e) summarizes the results. 
Though in theory it should in fact recognize the relative values of orientation before
and after picking up a passenger, and obtain the same solution as the soft-termination
policy, in practice we find that numerical imprecision actually causes a suboptimal policy to be
learned.
Because most of the rewards are negative per step, small differences in orientation
can be more difficult to distinguish amongst for an infinite discounted sum.
This result actually suggests that having multiple subgoals, as one might have with RL subtasks,
could enable better chaining of decisions and local evaluation of the optimal action.
The utility of learning with a smaller $\gamma_c$ has been previously
described \cite{jiang2015thedependence}, however, here we further advocate
that enabling $\gamma$ that provides subtasks is another strategy to improve learning. 

\section{Objectives and algorithms}
With an intuition for the specification of problems as RL tasks, 
we now turn to generalizing some key algorithmic concepts to enable learning for RL tasks.
We first generalize the definition of the Bellman operator for the value function. 
For a policy $\pi: \States \times \Actions \rightarrow [0,1]$,
define $\Ppi, \Pgam \in \RR^{\nstates \times \nstates}$ and $\rpi, \vvpi \in \RR^\nstates$,
indexed by states $s,s' \in \States$,
\small
\begin{align*}
&\Ppi(s,s') \defeq \sum_{a\in\Actions} \pi(s,a)\Pfcn(s,a,s') \\
&\Pgam(s,s') \defeq \sum_{a\in\Actions} \pi(s,a)\Pfcn(s,a,s') \gamma(s,a,s')\\
 & \rpi(s) \defeq \sum_{a\in\Actions} \pi(s,a) \sum_{s'\in\States} \Pfcn(s,a,s')\Rfcn(s,a,s')\\
  &\vvpi(s) 
\defeq \rpi(s) +  \sum_{s' \in \States} \Pgam(s,s')  \vvpi(s')
  .
\vspace{-0.3cm}
\end{align*}
\normalsize
%
%
where $\vvpi(s)$
is the expected return, starting from a state $s \in \States$.
To compute a value function that satisfies this recursion,
we define a Bellman operator.
%
The Bellman operator has been generalized to include state-based 
discounting and a state-based trace parameter\footnote{A generalization
to state-based trace parameters has been considered \citep{sutton1995td,sutton1998reinforcement,maei2010gq,sutton2014anew,yu2012least}.}
\citep[Eq. 29]{sutton2016anemphatic}. 
We further generalize the definition to the transition-based setting.
The trace parameter $\lambda: \States \times \Actions \times \States \rightarrow [0,1]$
influences the fixed point and provides a modified (biased) return, 
called the $\lambda$-return; this parameter is typically motivated as a bias-variance trade-off parameter \citep{kearns2000bias}. 
Because the focus of this work is on generalizing the discount, we opt for a simple constant $\lambda_c$
in the main body of the text; we provide generalizations to transition-based trace parameters in the appendix.

The generalized Bellman operator $\Tl:\RR^\nstates\rightarrow\RR^\nstates$
is
\begin{align}
 \Tl \vvec &\defeq \rl+\Pl \vvec, ~~~~\forall \vvec\in\RR^\nstates \label{eq_bellman}\\
\text{where } \ \ \ \ \ \Pl &\defeq\left(\eye - \lambda_c\Pgam\right)^{-1} \Pgam (1-\lambda_c)\\
\rl &\defeq \left(\eye-\lambda_c\Pgam\right)^{-1} \rpi \nonumber
\end{align}
%

To see why this is the definition of the Bellman operator,
we define the expected $\lambda$-return, $\vvlambda \in \RR^{\nstates}$
for a given approximate value function, given by a vector $\vvec \in \RR^{\nstates}$.
%
\begin{align*}
\!\!\!\vvlambda(s)  
&\defeq \rpi(s) \!+\!\! 
 \sum_{s' \in \States} \!\Pgam(s,s')
 \left[(1\!-\!\lambda_c)\vvec(s') \!+\! \lambda_c \vvlambda(s') \right]\\
&=  \rpi(s) + (1-\lambda_c)\Pgam(s,:) \vvec 
+\lambda_c\Pgam(s,:) \vvlambda
.
\end{align*}
Continuing the recursion, we obtain\footnote{For a matrix $\Mmat$ with maximum eigenvalue less than 1, $\sum_{i=0}^\infty \Mmat^i = (\eye - \Mmat)^{-1}$. We show in Lemma \ref{lem_eigs} that $\Pgam$ satisfies this condition, implying $\lambda_c\Pgam$ satisfies this condition and so this infinite sum is well-defined.}
\begin{align*}
\vvlambda
&= \left[\sum_{i=0}^\infty (\lambda_c\Pgam)^i \right]\left(\rpi + (1-\lambda_c)\Pgam {\vvec} \right)\\
&= (\eye - \lambda_c\Pgam)^{-1} \left(\rpi + (1-\lambda_c)\Pgam{\vvec} \right)
= \Tl \vvec
\end{align*}
The fixed point for this formula satisfies
$\Tl \vvec = \vvec$ for the Bellman operator defined in Equation \eqref{eq_bellman}.
%
%
%

To see how this generalized Bellman operator modifies the algorithms, we consider 
the extension to temporal difference algorithms. 
Many algorithms can be easily generalized by replacing $\gamma_c$ or
$\gamma_s(s_{t+1})$ with transition-based $\gamma(s_t, a_t, s_{t+1})$.
For example, the TD algorithm is generalized by setting the discount
on each step to $\gamma_{t+1} = \gamma(s_t,a_t,s_{t+1})$, 
\begin{align}
\wvec_{t+1} &= \wvec_t + \alpha_t \delta_t \evec_t \hspace{1.5cm} \triangleright \text{ for some step-size $\alpha_t$} \nonumber\\
\delta_t &\defeq r_{t+1} + \gamma_{t+1} \xvec(s_{t+1})^\top \wvec - \xvec(s_{t})^\top \wvec \label{eq_td}\\
\evec_t &= \gamma_t \lambda_c \evec_{t-1} + \xvec(s_t)\nonumber
.
\end{align}
%
The generalized TD fixed-point, under linear function approximation, can be expressed 
as a linear system
$\Amat \wvec = \bvec$
%
\begin{align*}
\Amat &= \Xmat^\top \Dmat (\eye - \lambda_c\Pgam)^\inv (\eye -  \Pgam)  \Xmat \\
\bvec &= \Xmat^\top \Dmat (\eye - \lambda_c\Pgam)^\inv \rpi
\end{align*}
where each row in $\Xmat \in \RR^{\nstates \times \xdim}$ corresponds to features for a state,
and $\Dmat \in \RR^{\nstates \times \nstates}$ is a diagonal weighting matrix.
Typically, $\Dmat = \diag(\dmu)$, where $\dmu \in \RR^{\nstates}$ is the stationary distribution for the behavior policy $\mu: \States \times \Actions \rightarrow [0,1]$ generating the stream of interaction. In on-policy learning, $\dmu = \dpi$.  
With the addition of the interest function, this weighting changes to $\Dmat = \diag(\dmu \hada \ivec)$,
where $\hada$ denotes element-wise product (Hadamard product). 
More recently, a new algorithm, emphatic TD (ETD) \citep{mahmood2015emphatic,sutton2016anemphatic} specified
yet another weighting, $\Dmat = \Mmatpi$ where $\Mmatpi = \diag(\mvec)$ with $\mvec = (\eye - \Pl)^{-1} (\dmu\hada\ivec)$.
Importantly, even for off-policy sampling, with this weighting, $\Amat$ is guaranteed to be positive definite. 
We show in the next section that the generalized Bellman operator for both the on-policy and emphasis weighting is a contraction,
and further in the appendix that the emphasis weighting with a transition-based trace function is also a contraction.  

\section{Generalized theoretical properties}\label{sec_proof}

In this section, we provide a general approach to incorporating transition-based discounting
into approximation bounds.
Most previous bounds have assumed a constant discount.
For example, ETD 
was introduced with state-based $\gamma_s$; however, 
\cite{hallak2015generalized} analyzed approximation error bounds of ETD
using a constant discount $\gamma_c$. 
By using matrix norms on $\Pgam$, we generalize previous approximation bounds
to both the episodic and continuing case.

Define the set of bounded vectors
for the general space of value functions
$\Fallv = \{ \vvec \in \RR^{\nstates}: \| \vvec \|_{\Dmu} < \infty\}$.
Let
$\Fsubspace \subset \Fallv$ be a subspace
of possible solutions, 
e.g., $\Fsubspace = \{ \Xmat \wvec |  \wvec \in \RR^\xdim, \| \wvec \|_2 < \infty\}$.


\newcommand{\asseigs}{A3}
\newcommand{\Pmu}{\mathbf{P}_{\mu}}

%
%
\begin{enumerate}
\item[A1.] The action space $\Actions$ and state space $\States$ are finite.
\item[A2.] 
For polices $\mu,\pi:\States \times \Actions \rightarrow[0,1]$,
there exist unique invariant distributions $\dmu,\dpi$ such that
$\dpi \Ppi = \dpi$ and $\dmu \Pmu = \dmu$.
This assumption is typically satisfied by assuming an ergodic Markov chain
for the policy. 
\item[A3.] 
There exist transition $s,a,s'$ such that $\gamma(s,a,s') < 1$
and $\pi(s,a)P(s,a,s') > 0$. This assumptions states that the policy 
reaches some part of the space where the discount is less than 1.
\item[A4.] 
Assume for any $\vvec \in \Fsubspace$, if $\vvec(s) = 0$ for all $s \in \States$
where $\ivec(s) > 0$, then $\vvec(s) = 0$ for all $s \in \States$ s.t. $\ivec(s) = 0$.
For linear function approximation, this requires
$\mathcal{F} = \text{span}\{ \xvec(s) : s \in \States, \ivec(s) \neq 0\}$. 
\end{enumerate}
For weighted norm $\| \vvec \|_{\Dmat} = \sqrt{\vvec^\top \Dmat \vvec}$, 
if we can take the square root of $\Dmat$,
the induced matrix norm is
$\| \Pl \|_{\Dmat} = \spnorm{ \Dmat^{1/2} \Pl \Dmat^{1/2} }$,
where the spectral norm $\spnorm{\cdot}$ is
the largest singular value of the matrix. 
For simplicity of notation below, define $\sbound{\Dmat} \defeq \| \Pl \|_{\Dmat}$.
For any diagonalizable, nonnegative matrix $\Dmat$, 
the projection $\Proj_{\Dmat}: \Fallv \rightarrow \Fsubspace$ onto $\Fsubspace$ exists
and is defined
$\Proj_{\Dmat}\zvec = \argmin_{\vvec \in \Fsubspace} \| \zvec - \vvec \|_\Dmat$.
 
 \subsection{Approximation bound}
 
 We first prove that the generalized Bellman operator
in Equation \ref{eq_bellman} is a contraction.
 We extend the bound from \citep{tsitsiklis1997ananalysis,hallak2015generalized} for
 constant discount and constant trace parameter,
 to the general transition-based setting.
The normed difference to the true value function 
could be defined
by multiple weightings. 
A well-known result is that for $\Dmat = \Dpi$
the Bellman operator is a contraction for constant $\gamma_c$ and $\lambda_c$ \citep{tsitsiklis1997ananalysis};
recently, this has been generalized for a variant of ETD to $\Mmat$, still with constant parameters \citep{hallak2015generalized}. 
We extend this result for transition-based $\gamma$ 
 for both $\Dpi$ and the
transition-based emphasis matrix $\Mmat$.
%
%
\newcommand{\normrow}{\boldsymbol{\xi}}

\begin{lemma}\label{lem_sbound}
For $\Dmat = \Dpi$ or $\Dmat = \Mmatpi$,
\begin{equation*}
\sbound{\Dmat} = \| \Pl \|_{\Dmat} < 1
.
\end{equation*}
\end{lemma}
  \begin{proof}
For $\Dmat = \Mmatpi$:
let $\normrow \in \RR^\nstates$ be the vector
of row sums for $\Pl$: $\Pl \onevec = \normrow$. 
Then for any $\vvec \in \Fallv$, with $\vvec \neq \zerovec$,
\begin{align*}
\| \Pl \vvec\|_{\Mmatpi}^2 &= \sum_{s \in \States} \mvec(s) \left( \sum_{s' \in \States} \Pl(s,s') \vvec(s')\right)^2\\
&= \sum_{s \in \States} \mvec(s) \normrow(s)^2 \left( \sum_{s' \in \States} \frac{\Pl(s,s')}{\normrow(s)} \vvec(s')\right)^2\\
&\le \sum_{s \in \States} \mvec(s) \normrow(s)^2  \sum_{s' \in \States} \frac{\Pl(s,s')}{\normrow(s)} \vvec(s')^2\\
&= \sum_{s' \in \States} \vvec(s')^2 \sum_{s \in \States} \mvec(s) \normrow(s) \Pl(s,s') \\
&= \vvec^\top  \diag\left( (\mvec \hada \normrow)^\top \Pl \right) \vvec
\end{align*}
where the first inequality follows from Jensen's inequality, because $\Pl(s,:)$ is normalized.
Now because $\normrow$ has entries that are less than 1, because the
row sums of $\Pl$ are less than 1 as shown in Lemma \ref{lem_probmatrix},
and because each of the values in the above product are nonnegative,
\begin{align*}
\vvec^\top & \diag\left( (\mvec \hada \normrow)^\top \Pl \right) \vvec\\
&\le \vvec^\top  \diag\left(\mvec^\top \Pl \right)  \vvec\\
&= \vvec^\top  \diag\left(\mvec^\top (\Pl - \eye) + \mvec^\top  \right)  \vvec\\
&= \vvec^\top  \diag\left(-(\dmu\hada\ivec)^\top + \mvec^\top  \right)  \vvec\\
&= \vvec^\top  \diag\left( \mvec^\top  \right)  \vvec
-  \vvec^\top  \diag\left((\dmu\hada\ivec)^\top  \right)  \vvec\\
&< \| \vvec \|_{\Mmatpi}^2
\end{align*}
The last inequality is a strict inequality because 
$\dmu\hada\ivec$ has at least one positive entry where $\vvec$ has a positive entry.
Otherwise, if $\vvec(s) = 0$ everywhere with $\ivec(s) > 0$, then $\vvec = \zerovec$,
which we assumed was not the case. 

Therefore, $\| \Pl \vvec\|_{\Mmat} <   \| \vvec \|_{\Mmat}$ for any $\vvec \neq \zerovec$,
giving 
$\| \Pl \|_{\Mmat} \defeq \max_{\vvec \in \RR^{\nstates}, \vvec \neq 0} \frac{\| \Pl \vvec\|_{\Mmat}}{\| \vvec \|_{\Mmat}}
< 1$. 
This exact same proof follows through verbatim for the generalization of $\Pl$
to transition-based trace $\lambda$. 

For $\Dmat = \Dpi$:
Again, we use Jensen's inequality, but now rely on the property $\dpi \Ppi = \dpi$.
Because of Assumption A3, for some $s < 1$, for any non-negative $\vvec_+$,
\begin{align*}
\dpi \Pgam \vvec_+ &=
\sum_{s} \sum_a \dpi(s) \Pfcn(s,a,:) \pi(s,a) \gamma(s,a,:) \vvec_+ \\
&\le s \sum_{s} \sum_a \dpi(s) \Pfcn(s,a,:) \pi(s,a)\vvec_+ 
= s\dpi \vvec
.
\end{align*}
Therefore, because vectors $\Pgam \vvec_+$ are also non-negative,
\begin{align*}
\dpi \Pl \vvec_+ &= \dpi \left(\sum_{k=0}^\infty (\Pgam \lambda_c)^k \Pgam (1-\lambda_c) \right) \vvec_+\\
&\le  (1-\lambda_c) \sum_{k=0}^\infty (s \lambda_c)^k \dpi\Pgam \vvec_+\\
&\le  (1-\lambda_c) (1-s\lambda_c)^\inv s\dpi \vvec_+
\end{align*}
\vspace{-0.2cm}
%
%
and so
\begin{align*}
\| \Pl \vvec\|_{\Dpi}^2 
&\le \sum_{s \in \States} \dpi(s) \normrow(s)^2  \sum_{s' \in \States} \frac{\Pl(s,s')}{\normrow(s)} \vvec(s')^2\\
&= \sum_{s' \in \States} \vvec(s')^2 \sum_{s \in \States} \dpi(s) \normrow(s) \Pl(s,s') \\
&\le \sum_{s' \in \States} \vvec(s')^2 \sum_{s \in \States} \dpi(s) \Pl(s,s') \\
&\le \tfrac{s (1-\lambda_c)}{1-\lambda_c s}\sum_{s' \in \States} \dvec(s') \vvec(s')^2 \\
&\le \tfrac{s- s\lambda_c}{1-\lambda_c s}\| \vvec \|_{\Dpi}^2
\end{align*}
where $\tfrac{s- s\lambda_c}{1-\lambda_c s} < 1$ since $s < 1$. 
 \end{proof}
 
 \begin{lemma}\label{theorem_bellman}
 Under assumptions A1-A3,
the Bellman operator $\Tl$ in Equation \eqref{eq_bellman} is a contraction
under a norm weighted by $\Dmat = \Dpi$ or $\Dmat = \Mmat_\pi$, i.e.,
for $\vvec \in \Fallv$ 
\begin{align*}
\| \Tl \vvec \|_{\Dmat} < \| \vvec \|_{\Dmat}
.
\end{align*}
Further, because the projection $\Proj_\Dmat$ is a contraction,
$\Proj_\Dmat \Tl$ is also a contraction and has a unique fixed point $\Proj_\Dmat \Tl \vvec = \vvec$
for $\vvec \in \Fsubspace$.
 \end{lemma}
 \begin{proof}
For any two vectors $\vvec_1, \vvec_2$
\begin{align*}
\| \Tl \vvec_1 - \Tl \vvec_2 \|_{\Dmat} 
&= \| \Pl (\vvec_1 - \vvec_2) \|_{\Dmat}\\
&\le
\| \Pl \|_{\Dmat} \| \vvec_1 - \vvec_2 \|_{\Dmat}\\
&<  \| \vvec_1 - \vvec_2 \|_{\Dmat}
\end{align*}
where the last inequality follows from Lemma \ref{lem_sbound}.
 By the Banach Fixed Point theorem, because
 the Bellman operator is a contraction under $\Dmat$,
 it has a unique fixed point.
 \end{proof}
  
 \begin{theorem}\label{theorem_projsolution}
If $\Dmat$ satisfies $\sbound{\Dmat} < 1$, 
then there exists
$\vvec \in \Fsubspace$ such that $\Proj_\Dmat \Tl \vvec = \vvec$,
and 
the error to the true value function
is bounded as
 \begin{align}
 \| \vvec - \vvec^* \|_{\Dmat} \le (1-\sbound{\Dmat})^{-1}  \| \Proj_\Dmat \vvec^* - \vvec^* \|_{\Dmat}
 .
 \end{align}
 For constant discount $\gamma_c \in [0,1)$ and constant trace parameter 
 $\lambda_c \in [0,1]$ with $\Dmat = \Dpi$, this bound reduces to the original bound
 \citep[Lemma 6]{tsitsiklis1997ananalysis}:
 $$(1-\sbound{\Dmat})^{-1} \le \frac{1-\gamma_c \lambda_c}{1-\gamma_c}.$$
 \end{theorem}
 \begin{proof} 
Let $\vvec$ be the
unique fixed point of $\Proj_\Dmat \Tl$,
which exists by Lemma \ref{theorem_bellman}.
 \begin{align*}
\| \vvec - \vvec^* \|_{\Dmat} &\le  \| \vvec - \Proj_\Dmat \vvec^* \|_{\Dmat} +  \| \Proj_\Dmat \vvec^* - \vvec^* \|_{\Dmat}\\
 &=  \| \Proj_\Dmat \Tl \vvec - \Proj_\Dmat \vvec^* \|_{\Dmat} +  \| \Proj_\Dmat \vvec^* - \vvec^* \|_{\Dmat}\\
 &\le  \| \Tl \vvec -  \vvec^* \|_{\Dmat} +  \| \Proj_\Dmat \vvec^* - \vvec^* \|_{\Dmat}\\
 &=  \| \Tl (\vvec -  \vvec^*) \|_{\Dmat} +  \| \Proj_\Dmat \vvec^* - \vvec^* \|_{\Dmat}\\
 &=  \| \Pl (\vvec -  \vvec^*) \|_{\Dmat} +  \| \Proj_\Dmat \vvec^* - \vvec^* \|_{\Dmat}\\
 &\le  \|\Pl\|_{\Dmat} \| \vvec -  \vvec^* \|_{\Dmat} +  \| \Proj_\Dmat \vvec^* - \vvec^* \|_{\Dmat}\\
 &=  \sbound{\Dmat} \| \vvec -  \vvec^* \|_{\Dmat} +  \| \Proj_\Dmat \vvec^* - \vvec^* \|_{\Dmat}
 \end{align*}
 where the second inequality is due to $\| \Proj_\Dmat \Tl \vvec \|_{\Dmat} \le \| \Tl \vvec \|_{\Dmat}$,
  the second equality due to $\Tl \vvec^* = \vvec^*$
  and the third equality due to $\Tl \vvec - \Tl\vvec^* = \Pl(\vvec - \vvec^*)$
  because the $\rpi$ cancels.
 By rearranging terms, we get
  \begin{align*}
(1-\sbound{\Dmat}) \| \vvec - \vvec^* \|_{\Dmat} &\le \| \Proj_\Dmat \vvec^* - \vvec^* \|_{\Dmat}
 \end{align*}
 and since $\sbound{\Dmat} < 1$, we get the final result.
 
For constant 
 $\gamma_c < 1$ and $\lambda_c$, we know that $\Pgam = \gamma\Ppi$. Further, if $\Dmat = \Dpi$, we know $\|\Ppi^{i+1} \|_{\Dmat} = 1$. Therefore,
   \begin{align*}
\sbound{\Dmat} &= \|\Pl\|_{\Dmat} \\
&=  \| \Dmat^{1/2} \left(\sum_{i=0}^\infty \gamma_c^i \lambda_c^i \Ppi^i \right) \gamma_c (1-\lambda_c) \Ppi \Dmat^{-1/2} \|_{2}\\
&\le  \gamma_c (1-\lambda_c)   \sum_{i=0}^\infty \gamma_c^i \lambda_c^i  \|\Dmat^{1/2} \Ppi^{i+1} \Dmat^{-1/2} \|_{2}\\
&=  \gamma_c (1-\lambda_c) \sum_{i=0}^\infty \gamma_c^i \lambda_c^i  \|\Ppi^{i+1} \|_{\Dmat}\\
&\le  \gamma_c (1-\lambda_c)  \sum_{i=0}^\infty \gamma_c^i \lambda_c^i \\
&=  \frac{\gamma_c (1-\lambda_c)}{1-\gamma_c\lambda_c}
 \end{align*}
 \palignbox
 \end{proof}
 

 We provide generalizations to transition-based trace parameters in the appendix,
for the emphasis weighting, and also discuss issues with generalizing to state-based termination for a standard weighting with $\dpi$. 
We show that for any transition-based discounting function $\lambda: \States \times \Actions \times \States \rightarrow [0,1]$, the above contraction
results hold under emphasis weighting. 
We then provide a general form for an upper bound on $\| \Pl \|_{\Dpi}$ for transition-based discounting, based on the contraction properties of two matrices within $\Pl$. 
We further provide an example where the Bellman operator is not a contraction even under the simpler generalization
to state-based discounting,
and discuss the requirements for the transition-based generalizations to ensure a contraction with weighting $\dpi$. 
This further motivates the emphasis weighting as a more flexible scheme
for convergence under general setting---both off-policy and transition-based generalization.

\subsection{Properties of TD algorithms}

Using this characterization of $\Pl$, we can re-examine 
previous results for temporal difference algorithms
that either used state-based or constant discounts. 

\paragraph{Convergence of Emphatic TD for RL tasks.}
We can extend previous convergence results for ETD, for learning value functions and action-value functions,
for the RL task formalism. 
For policy evaluation, 
ETD and ELSTD, the least-squares version of ETD that uses the above defined $\Amat$ and $\bvec$ with $\Dmat = \Mmatpi$,
have both been shown to converge
with probability one \citep{yu2015onconvergence}. As an important component of this proof is convergence in expectation,
which relies on $\Amat$ being positive definite. 
In particular, for appropriate step-sizes $\alpha_t$ (see \citep{yu2015onconvergence}),
if $\Amat$ is positive definite,  
the iterative update is convergent
%
$\wvec_{t+1} = \wvec_t + \alpha_t (\bvec -\Amat \wvec_t)$.
For the generalization to transition-based discounting, convergence in expectation extends for the emphatic algorithms.
We provide these details in the appendix for completeness, 
with theorem statement and proof in Appendix \ref{app_proof}
and pseudocode in Appendix \ref{app_algorithms}.

\newcommand{\Pbeta}{{\mathbf{P}^{\lambda,\beta}_\pi}}


\paragraph{Convergence rate of LSTD($\lambda$).}
\citet{tagorti2015ontherate} recently provided convergence rates for LSTD($\lambda$) for continuing tasks,
for some $\gamma_c < 1$. These results can be extended to the episodic setting
with the generic treatment of $\Pl$. For example, in \citep[Lemma 1]{tagorti2015ontherate},
which describes the sensitivity of LSTD, the proof extends
by replacing the matrix $(1-\lambda_c) \gamma_c \Ppi (\eye - \lambda_c \gamma_c \Ppi)^\inv$
(which they call $M$ in their proof)
with the generalization $\Pl$, resulting instead in the constant $\frac{1}{1-\sbound{\Dmat}}$ in the bound rather
than $\frac{1-\lambda_c \gamma_c}{1-\gamma_c}$.
Further, this generalizes convergence rate results to emphatic LSTD, since $\Mmat$
satisfies the required convergence properties, with rates dictated by $\sbound{\Mmat}$
rather than $\sbound{\Dmu}$ for standard LSTD.

\paragraph{Insights into $\sbound{\Dmat}$.}
Though the generalized form enables unified episodic and continuing results,
the resulting bound parameter $\sbound{\Dmat}$ is more difficult to interpret
than for constant $\gamma_c, \lambda_c$. With $\lambda_c$ increasing
to one, the constant $\frac{1-\gamma_c \lambda_c}{1-\gamma_c}$
in the upper bound decreased to one. For $\gamma_c$ decreasing to zero, 
the bound also decreases to one. These trends are intuitive, as the problem should
be simpler when $\gamma_c$ is small, and bias should be less when $\lambda_c$ is close to one. 
More generally, however, the discount can be small or large for different transitions,
making it more difficult to intuit the trend. 

To gain some intuition for $\sbound{\Dmat}$, consider 
a random policy in the taxi domain, with $\sbound{\Dmat}$ summarized in Table \ref{table_sd}.
As $\lambda_c$ goes to one,
$\sbound{\Dmat}$ goes to zero and so $(1-\sbound{\Dmat})^\inv$ goes to one. 
Some outcomes of note are that
1) hard or soft termination for the pickup results in the exact same $\sbound{\Dmat}$;
2) for a constant gamma of $\gamma_c = 0.99$, the episodic discount had a slightly smaller $\sbound{\Dmat}$; and
3) increasing $\lambda_c$ has a much stronger effect than including more terminations.
Whereas, when we added random terminations, so that from 1\% and 10\% of the states, termination occurred on at least one path  within 5 steps or even more aggressively on every path within 5 steps, the values of $\sbound{\Dmat}$ were similar.  

\begin{table}[h!]
{
\centering
\setlength\tabcolsep{4pt}
\begin{sc}
\begin{tabular}{|c|c|c|c|c|c|}
\hline
{\scriptsize $\lambda_c$} & {\scriptsize 0.0} & {\scriptsize 0.5} & {\scriptsize 0.9} & {\scriptsize 0.99} & {\scriptsize 0.999}\\
\hline
  {\scriptsize Episodic taxi}  & 0.989 & 0.979 & 0.903 & 0.483 & 0.086\\
  {\scriptsize $\gamma_c = 0.99$} & 0.990 & 0.980 & 0.908 & 0.497 & 0.090\\
  \hline
  {\scriptsize 1\% single path} & 0.989 & 0.978 & 0.898 & 0.467 & 0.086\\
  {\scriptsize 10\% single path} & 0.987 & 0.975 & 0.887 & 0.439 & 0.086\\
  \hline
  {\scriptsize 1\% all paths} & 0.978  & 0.956 & 0.813 & 0.304 & 0.042\\
  {\scriptsize 10\% all paths} & 0.898 & 0.815 & 0.468 & 0.081 & 0.009\\
\hline
\end{tabular}
\end{sc}
}
\caption{ The $\sbound{\Dmat}$ values for increasing $\lambda_c$, with discount settings described in the text. }\label{table_sd}
\end{table}

\section{Discussion and conclusion}

%

The goal of this paper is to provide intuition and examples of
how to use the {\em RL task formalism}. Consequently, to avoid jarring the explanation, technical contributions were not emphasized, and in some cases included only in the appendix.
For this reason, we would like to highlight and summarize the technical
contributions, which include
1) the introduction of the RL task formalism, and of transition-based discounts;
2) an explicit characterization of the relationship between state-based and transition-based discounting; and
3) generalized approximation bounds, applying
to both episodic and continuing tasks;
and 4) insights into---and issues with---extending contraction results for both state-based and transition-based discounting. 
Through intuition from simple examples 
and fundamental theoretical extensions, this work
provides a relatively complete characterization of the RL task formalism, 
as a foundation for use in practice and theory. 

\section*{Acknowledgements}
Thanks to Hado van Hasselt for helpful discussions about transition-based discounting,
and probabilistic discounts. 
{
\bibliographystyle{plainnat}
\bibliography{paper.bib}
}

\newpage
\appendix

\section{More general formulation with probabilistic discounts}\label{app_probabilistic}

In the introduction of transition-based discounting, we
could have instead assumed that we had a more general
probability model: $\Pfcn(r, \gamma | s, a, s')$.
Now, both the reward and discount are not just functions
of states and action, but also are stochastic. This generalization in fact,
does not much alter the treatment in this paper.
This is because, when taking the expectations for value function,
the Bellman operator and the $\Amat$ matrix, we are left again
with $\gamma(s,a,s')$. To see why,
\begin{align*}
\vpi(s)  
&=\sum_{a,s'}\pi(s,a) \Pfcn(s,a,s') \text{E}[ r + \gamma \vpi(s') | s, a, s'] \\
&= \sum_{a,s'}\pi(s,a) \Pfcn(s,a,s') \text{E}[r | s, a, s'] \\
&\ \ \ \ + \sum_{a,s'}\pi(s,a) \Pfcn(s,a,s') \text{E}[\gamma | s, a, s'] \vpi(s')\\
&= \rpi(s) 
+ \sum_{s'} \Pgam(s,s')\vpi(s')  
\end{align*}
for $\gamma(s,a,s') = \text{E}[\gamma | s, a, s']$.
 

\section{Relationship between state-based and transition-based discounting}

In this section, 
we show that for any MDP with
transition-based discounting, we can construct an equivalent MDP
with state-based discounting. The MDPs are equivalent in the sense
that learned policies and value functions learned in either MDP
would have equal values when evaluated on the states
in the original transition-based MDP. This equality ignores practicality
of learning in the larger induced state-based MDP, and at
the end of this section, we discuss advantages of the more compact
transition-based MDP.

\subsection{Equivalence result}\label{sec_equivalence}

The equivalence is obtained by introducing \fake\ states
for each transition. The key is then to prove that
%
%
the stationary distribution
for the state-based MDP, with additional \fake\ states,
provides the same solution even with function approximation.
For each triplet $s,a,s'$, add a new \fake\ state $\fakes_{sas'}$,
with set $\fakeset$ comprised of these additional states. Each transition now goes through a \fake\ state, $\fakes_{sas'}$, and allows the discount in the \fake\ state to be set to $\gamma(s,a,s')$. 
The induced state-based MDP has state set $\fakeS = \States \cup \fakeset$ with $| \fakeS | = |\Actions|\nstates^2 + \nstates$.
We define the other models in the proof in Appendix \ref{app_equivalence}. 

The choice of action in the \fake\ states is irrelevant. To extend the policy $\pi$,
we arbitrarily choose that the policy uniformly selects actions when in the \fake\ states
and define $\fakepi(s,a) = \pi(s,a)$ for $s \in \States$ and $\fakepi(s,a) = 1/|\Actions|$ otherwise. 
%
%
For linear function approximation, 
we also
need to 
assume $\xvec(f_{sas'}) = \xvec(s')$ for $f_{sas'} \in \fakeset$.

\begin{theorem}\label{theorem_equivalence}
For a given transition-based MDP 
$(\Pfcn, \Rfcn, \States, \Actions, \gamma)$ and policy $\pi$,
assume that the stationary distribution 
$\dpi$ exists. 
Define state-based MDP $(\fakeP, \fakeR, \fakeS, \Actions, \fakegam)$
with extended $\fakepi$, all as above.
Then the stationary distribution $\fakedpi$ for $\fakepi$ exists and satisfies
\begin{align}
\tfrac{\fakedpi(s)}{\sum_{i \in \States} \fakedpi(i)} = \dpi(s) \label{eq_stationary}
.
\end{align}
$\forall s \in \States$, $\fakevpi(s) = \vpi(s)$ and
 $\fakepi(s,a) = \pi(s,a)$ for all $s \in \States, a \in \Actions$ with
%
$\pi = \argmin_{\pi} \sum_{s \in \States} \dpi(s) \vpi(s); \ \ \fakepi   = \argmin_{\pi} \sum_{i \in \fakeS} \fakedpi(s) \fakevpi(s)$
\end{theorem}

\subsection{Advantages of transition-based discounting over state-based discounting}\label{sec_advantages}

Though the two have equal representational abilities,
there are several 
disadvantages of state-based discounting that compound
to make the more general transition-based discount strictly more desirable.
The disadvantages of using an induced state-based MDP, rather than
the original transition-based MDP, arises from
the addition of states and include the following. 

\textbf{Compactness.} In the worst-case, 
for a transition-based MDP with $\nstates$ states, 
the induced state-based MDP can have
$|\Actions|\nstates^2+n$ states. 

\textbf{Problem definition changes for different discounts.}
For the same underlying MDP, multiple learners with different discount
functions would have different induced state-based MDPs.
This complicates code and reduces opportunities for sharing
variables and computation. 

\textbf{Overhead.} Additional states must be stored, with additional algorithmic updates in those non-states, or cases to avoid these updates,
and the need to carefully set features for \fake\ states. This overhead is both computational 
as well as conceptual, as it complicates the code. 

\textbf{Stationary distribution.} This distribution superfluously includes \fake\ states
and requires renormalization to obtain
the stationary distribution for the original transition-based MDP.

\textbf{Off-policy learning.} In off-policy learning, one goal is to learn
many value functions with different discounts \citep{white2015thesis}.
As mentioned above, these learners may have different induced state-based MDPs,
which complicates implementation and even theory. 
To theoretically characterize a set
of off-policy learners, it would be necessary to consider different induced state-based MDPs.
Further, sharing information,
such as the features, is again complicated by using induced state-based MDP
rather than a single transition-based MDP, with varying discount functions. 

\textbf{Specification of algorithms.}
Often algorithms are introduced either for the episodic case (e.g., true-online TD \citep{vanseijen2014true}) or
the continuing case (e.g., the lower-variance version of ETD \citep{hallak2015generalized}).
When kept separately, with explicit loops over episodes, the algorithm itself is different (e.g., Sarsa \citep[Figure 8.8]{sutton1998reinforcement}); or, if a state-based approach is used, fake states and fake transitions would have to be explicitly
added to make the update the same for continuing or episodic. 
For the generalized formulation, the only difference is the $\gamma_{t+1}$ that is passed to the
algorithm; the algorithm itself remains exactly the same in the two settings.  
As a minor example, for episodic problems, there is typically an explicit (error-prone) step to clear traces;
with generalized discounting, the traces are automatically cleared at the end of an episode by $\gamma_{t+1}$.

\textbf{Experimental design.}
When presenting results for episodic and continuing problem, often the former uses number of episodes and the later number of steps. In reality both simply consist of a trajectory of information, with the former having $\gamma_{t+1} = 0$ for some steps. A unified view with number of steps enables more consistent presentation of results across domains.
Related to this difference,
a common but rarely discussed decision when implementing episodic tasks is the cut-off for the maximum number of steps in an episode. If set too small, an algorithm that takes longer to reach the goal in the first few episodes, but then learns more quickly afterwards, could be unfairly penalized. Instead learning could be limited to some maximum number of steps to constrain learning time similarly for both continuing and episodic problems, where multiple episodes could occur within that maximum number of steps. 

\subsection{Proof of Theorem \ref{theorem_equivalence}}\label{app_equivalence}

This prove illustrates the representability relationship between
transition-based discounting and state-based discounting. 
This equivalence could be obtained more compactly if $\gamma(s,a,s')$ is not different for every $s'$; however,
the proof becomes much more involved. Since our main goal is to simply show a representability result, we opt for interpretability.
Note that, in addition,
the result in Theorem \ref{theorem_equivalence} fills a gap in the previous theory, which indicated
that state-based discounting could be used to represent episodic problems,
but did not explicitly demonstrate that the stationary distribution
would be equivalent (see \citep{yu2015onconvergence}).

Define
transition probabilities $\fakeP : \fakeS \times \fakeS \rightarrow [0,1]$
\begin{align*}
\fakeP(i,a,j) = \left\{ \begin{array}{ll}
        \Pfcn(i,a,s') & \mbox{$i \in \States$, $j = \fakes_{ias'}$}\\
        1 & \mbox{$j \in \States$, $i = \fakes_{saj}$}\\
         0 & \mbox{otherwise}
        \end{array} \right.
\end{align*}
rewards
\begin{align*}
\fakeR(i,a,j) = \left\{ \begin{array}{ll}
        \Rfcn(i,a,s') & \mbox{$i \in \States$,  $j = \fakes_{ias'}$}\\
         0 & \mbox{otherwise}
        \end{array} \right.
\end{align*}
and state-based discount function $\fakegam: \fakeS \rightarrow [0,1]$
\begin{align*}
\fakegam(i) = \left\{ \begin{array}{ll}
        \gamma(s,a,s') & \mbox{$i = \fakes_{sas'}$}\\
         1 & \mbox{otherwise}
        \end{array} \right.
\end{align*}
%

\textbf{Theorem \ref{theorem_equivalence}}
For a given transition-based MDP 
$(\Pfcn, \Rfcn, \States, \Actions, \gamma)$ and policy $\pi$,
assume that the stationary distribution 
$\dpi$ exists. 
Define state-based MDP $(\fakeP, \fakeR, \fakeS, \Actions, \fakegam)$
with extended $\fakepi$, all as above.
Then the stationary distribution $\fakedpi$ for $\fakepi$ exists and satisfies
\begin{align}
\frac{\fakedpi(s)}{\sum_{i \in \States} \fakedpi(i)} = \dpi(s) \label{eq_stationary}
\end{align}
and for all $s \in \States$, $\fakevpi(s) = \vpi(s)$.
\begin{proof}
Define matrix $\fakeppi \in \RR^{(\nstates + |\Actions| \nstates^2) \times (\nstates + |\Actions|\nstates^2)}$
where $\fakeppi(i,j) = \sum_{a \in \Actions} \fakepi(i,a) \fakeP(i,a,j)$, giving
\begin{align*}
\fakeppi(i,j) = \left\{ \begin{array}{ll}
       \pi(i,a)\Pr(i,a,s') & \mbox{$i \in \States, j = f_{ias'}$}\\
        1 & \mbox{$i = f_{saj}, a\in \Actions, j \in \States$}\\
          0 & \mbox{otherwise}
        \end{array} \right.
\end{align*}
%
%
Define
\begin{align*}
\fakedpi(i) \defeq \frac{1}{c} \left\{ \begin{array}{ll}
        \dpi(i) & \mbox{$i \in \States$}\\
         \dpi(s) \pi(s,a)\Pr(s,a,s')  & \mbox{$i = \fakes_{sas'}$}
        \end{array} \right.
\end{align*}
where $c > 0$ is a normalizer to ensure that 
$\onevec^\top\fakedpi = 1$. 
Now we need to show that
$\fakedpi \fakeppi= \fakedpi$. 
For any $j \in \States$,
\begin{align*}
\fakedpi \fakeppi(:,j) &= \frac{1}{c} \Big( \sum_{s \in \States} \dpi(s) \fakeppi(s,j) 
+ \sum_{\fakes \in \fakeset} \fakedpi(\fakes) \fakeppi(\fakes,j)\Big)
\end{align*}
\textbf{Case 1:} $j \in \States$\\
For the first component, because $\fakeppi(s,j) = 0$ by definition of $\fakeP$, we get
\begin{align*}
\sum_{s \in \States} \dpi(s) \fakeppi(s,j) &= 0
\end{align*}
For the second component, because $\fakeppi(\fakes_{saj},j) = 1$, 
\begin{align*}
\sum_{\fakes_{sas'} \in \fakeset} & \fakedpi(\fakes_{sas'}) \fakeppi(\fakes_{sas'},j) \\
&= \sum_{\fakes_{saj} \in \fakeset} \fakedpi(\fakes_{saj}) \fakeppi(\fakes_{saj},j)\\
&= \sum_{\fakes_{saj} \in \fakeset} \fakedpi(\fakes_{saj}) \\
&= \sum_{s \in \States} \dpi(s) \sum_{a\in \Actions} \pi(s,a) \Pfcn(s,a,j) \\
&= \sum_{s \in \States} \dpi(s) \Ppi(s,j) \\
&=  \dpi(j) 
\end{align*}
where the last line follows from the definition of the stationary distribution.
Therefore, for $j \in \States$
\begin{align*}
\fakedpi \fakeppi(:,j) &= \frac{1}{c} \dpi(j) = \fakedpi(j)
\end{align*}
\textbf{Case 2:} $j = f_{sas'} \in \fakeset$\\
For the first component, because $\fakeppi(i,f_{sas'}) = 0$ for all $i \neq s$
and because $\fakeppi(s,f_{sas'}) = \pi(s,a) \Pfcn(s,a,s')$ by construction,
\begin{align*}
\sum_{i \in \States} \dpi(i) \fakeppi(i,f_{sas'}) &= \dpi(s) \fakeppi(s,f_{sas'})  \\
&= \dpi(s) \pi(s,a) \Pfcn(s,a,s')\\
&= c \ \ \fakedpi(f_{sas'})
.
\end{align*}
For the second component,
because $\fakeppi(f,j) = 0$ for all $f,j \in \fakeset$, we get
\begin{align*}
\sum_{f \in \fakeset} \fakedpi(f) \fakeppi(f,j) = 0
.
\end{align*}
Therefore, for $j = f_{ss'} \in \fakeset$, 
$\fakedpi \fakeppi(:,j) = \fakedpi(j)$.

Finally, clearly by normalizing the first component of $\fakedpi$
over $s \in \States$, we get the same proportion across states
as in $\dpi$, satisfying \eqref{eq_stationary}.

To see why $\fakevpi(s) = \vpi(s)$ for all $s \in \States$,
first notice that
\begin{align*}
\fakerpi(i) = \left\{ \begin{array}{ll}
        \rpi(i) & \mbox{$i \in \States$}\\
          0 & \mbox{otherwise}
        \end{array} \right.
\end{align*}
and for any $f_{sas'} \in \fakeset$
\begin{align*}
\fakevpi(f_{sas'}) &= 0 + \sum_{j \in \fakeS} \fakeppi(f_{sas'},j) \fakegam(j) \fakevpi(j)\\
&= \fakevpi(s')
.
\end{align*}
Now for any $s \in \States$, 
\begin{align*}
\fakevpi(s) &= \fakerpi(s) +\sum_{j \in \fakeS} \fakeppi(s,j) \fakegam(j) \fakevpi(j)\\
&= \rpi(s) +\sum_{f_{sas'} \in \fakeset} \fakeppi(s,f_{sas'}) \fakegam(f_{sas'}) \fakevpi(f_{sas'})\\
&= \rpi(s) +\sum_{s' \in \States} \sum_{a \in \Actions} \Pfcn(s,a,s')\gamma(s,a,s') \fakevpi(s')
\end{align*}
Therefore, because it satisfies the same fixed point equation,
$\fakevpi(s) = \vpi(s)$ for all $s \in \States$.

With this equivalence, it is clear that 
\begin{align*}
&\sum_{i \in \fakeS} \fakedpi(i) \fakevpi(i) \\
&= \frac{1}{c} \sum_{s \in \States} \dpi(s) \vpi(s) 
+ \frac{1}{c} \sum_{f_{ss'} \in \fakeset} \dpi(s) \Ppi(s,s') \vpi(s') \\
&= \frac{1}{c} \sum_{s \in \States} \dpi(s) \vpi(s) 
+ \frac{1}{c} \sum_{s\in \States} \sum_{s' \in \States} \dpi(s) \Ppi(s,s') \vpi(s') \\
&= \frac{1}{c} \sum_{s \in \States} \dpi(s) \vpi(s) 
+ \frac{1}{c}  \sum_{s' \in \States} \dpi(s') \vpi(s') \\
&= \frac{2}{c} \sum_{s \in \States} \dpi(s) \vpi(s) 
\end{align*}
Therefore, optimizing either results in the same
policy.
\end{proof}

%
%

\section{Discounting and average reward for control}\label{sec_avereward}

The common wisdom is that discounting is useful for asking
predictive questions, but for control, the end goal is average reward.
One of the main reasons for this view is that it
has been previously shown that, for a constant discount,
optimizing the expected return is equivalent to optimizing average reward.
This can be easily seen by expanding the expected return 
weighting according to the stationary distribution for a policy,
given constant discount $\gamma_c < 1$, 
\begin{align}
\dpi \vvpi &= \dpi (\rpi + \Pgam \vvpi)\\
&= \dpi \rpi + \gamma_c \dpi \Ppi \vvpi \nonumber\\
&= \dpi \rpi + \gamma_c \dpi \vvpi \nonumber\\
\implies
\dpi \vvpi &= \frac{1}{1-\gamma_c}\dpi \rpi \label{eq_average}
.
\end{align}
Therefore, the constant $\gamma_c < 1$ simply scales
the average reward objective, so optimizing either provides the same policy.
This argument, however, does not extend to transition-based discounting, because 
$\gamma(s,a, s')$ can significantly change weighting in returns in a non-uniform way,
affecting the choice of the optimal policy. We demonstrate this in the case study for the taxi domain in Section \ref{sec_taxi}. 

\section{Algorithms}\label{app_algorithms}

\newcommand{\vold}{v_{\text{old}}}

We show how to write generalized pseudo-code for two algorithms:
true-online TD ($\lambda$) and ELSTDQ($\lambda$).
We choose these two algorithms because they generally demonstrate
how one would extend to transition-based $\gamma$,
and further previously had a few unclear points in their implementation.
For TO-TD, the pseudo-code has been given for episodic tasks \cite{vanseijen2014true}, rather than more generally,
and has treated $\vold$ carefully at the beginning of episodes, which is not necessary.
LSTDQ has typically only been written for a batch of data, without importance sampling;
we provide an ELSTDQ variant here with importance sampling,
where LSTDQ is a special case using $M = 1$.

\newcommand{\vhat}{\hat{v}}
\newcommand{\stepsize}{\alpha}

There are a few other implementation details that merit clarification.
We use the notation $\gamma_{t+1}$ for $\gamma(s_t,a_t,s_{t+1})$, and $\lambda_{t+1}$ for $\lambda(s_t, a_t, s_{t+1})$.
Further, unlike previous pseudo-code, we do not reinitialize $\vold$ specially at the start of an episode 
(i.e., when $\gamma_{t+1} = 0$). This is because the value of $\vold$ is not relevant for
the next step after $\gamma_{t+1} = 0$. The eligibility trace is zeroed, and so 
$\stepsize (\delta + \vhat - \vold) \evec - \stepsize (\vhat - \vold) \xvec
= \stepsize \delta\xvec$.
Finally, for both algorithms, we stage the updates to the traces. 
This is to avoid saving both $\gamma_t, \lambda_t$ and $\gamma_{t+1}, \lambda_{t+1}$ across timesteps.

\begin{algorithm}[t]
\caption{True-online TD($\lambda$)}
\label{alg_elstdq}
\begin{algorithmic}
  \State  $\wvec \gets \zerovec, \ \ \ \evec \gets \zerovec, \ \ \ \ \vold \gets 0$
  \State Obtain initial $\xvec_0$
\While{agent interacting with environment, $t = 0, 1, \ldots$}
\State Obtain next feature vector $\xvec_{t+1}$, reward $r_{t+1}$ and discount $\gamma_{t+1}$
\State $\vhat = \wvec^\top \xvec_t$
\State $\vhat' = \wvec^\top \xvec_{t+1}$
\State $\delta \gets r_{t+1} + \gamma_{t+1} \vhat' - \vhat$ 
\State $\evec \gets \evec + \xvec_t$
\State $\wvec \gets \wvec + \stepsize (\delta + \vhat - \vold) \evec - \stepsize (\vhat - \vold) \xvec_t$
\State $ \vold \gets \vhat'$
\State $\evec \gets \gamma_{t+1} \lambda_{t+1} \evec - \stepsize \gamma_{t+1} \lambda_{t+1} (\evec^\top \xvec_{t+1}) \xvec_{t+1}$
\EndWhile\\
\Return $\wvec$
\end{algorithmic}
\end{algorithm}

\begin{algorithm}
\caption{ELSTDQ($\lambda$)}
\label{alg_elstdq}
\begin{algorithmic}
  \State  $\Amat \gets \zerovec, \ \ \ \ \bvec \gets \zerovec, \ \ \ \evec \gets \zerovec$
  \State $F \gets 0, \ \ \  M \gets 0$
    \State Obtain initial action-value feature vector $\xvec_0$ (implicitly $\xvec(s_0,a_0)$) and action $a_0$
\While{agent interacting with environment, $t = 0, 1, \ldots$}
\State Obtain next action-value feature vector $\xvec_{t+1}$, action $a_{t+1}$ reward $r_{t+1}$ and discount $\gamma_{t+1}$
\State $ \rho_{t+1} \gets \frac{\pi(\xvec_{t+1},a_{t+1})}{\mu(\xvec_{t+1},a_{t+1})}$
\State $F \gets  \gamma_{t+1} F + i(\xvec_t)$
\State $M \gets \lambda_{t+1} i(\xvec_t) + (1-\lambda_{t+1}) F$
\State $\evec\gets \evec + M\xvec_t$
\State $\Amat \gets \Amat + \evec \left( \xvec_t - \rho_{t+1}\gamma_{t+1} \xvec_{t+1}\right)^\top$
\State $ \bvec \gets \bvec + \evec r_{t+1}$
\State $\evec\gets \gamma_{t+1}\lambda_{t+1} \rho_{t+1}\evec$
\State $F \gets \rho_{t+1} F$
\EndWhile\\
\Return $\Amat^{-1} \bvec$ \ \ \ \ \ \ \ \ \ // The solution $\wvec$ to the linear system
\end{algorithmic}
\end{algorithm}

%

\section{Lemmas}

To bound the maximum eigenvalues of the discount-weighted
transition matrix, we first provide the following lemma.
This lemma is independently interesting,
in that it explicitly verifies the previous Assumption 2.1 \citep{yu2015onconvergence}. 

\begin{lemma}~\label{lem_eigs}
Under Assumption \asseigs, the maximum eigenvalue of $\Pgam$ is less than 1:
\begin{align*}
\maxlam(\Pgam) < 1
.
\end{align*}
\end{lemma}
\begin{proof}

\noindent
\textbf{Part 1:} We first show that $\maxlam(\tilde{\Mmat}) < 1$ for 
\begin{align*}
\tilde{\Mmat} = \left\{ \begin{array}{ll}
         \Mmat_{kl} - \delta & \mbox{if $i =k, j = l$}\\
        \Mmat_{kl} & \mbox{otherwise}.\end{array} \right. 
\end{align*} 
for any $i,j$, and any $0 <  \delta < \Mmat_{ij}$.

For $\Mmat = \Ppi$, we know that $\maxlam(\Mmat) = 1$ 
and that, by assumption, $\Ppi$ is irreducible.
We know that $\tilde{\Mmat}$ is still irreducible, because 
the connectivity is not changed (since no additional entries are zeroed). 
By the Perron-Frobenius theorem, we know that the eigenvector $\xvec$ that
corresponds to the maximum eigenvalue
$\maxlam(\tilde{\Mmat})$ has strictly positive entries
and so
$\delta \xvec_i  \xvec_j > 0$.
Therefore,
\begin{align*}
\xvec^\top \tilde{\Mmat} \xvec &= \sum_{k, l: (k,l) \neq (i,j)} \xvec_k \Mmat_{kl} \xvec_l + \xvec_i (\Mmat_{ij} - \delta) \xvec_j\\
&= \xvec^\top {\Mmat} \xvec - \delta \xvec_i  \xvec_j\\
&< \xvec^\top {\Mmat} \xvec
.
\end{align*}
%
We know that $\xvec^\top {\Mmat} \xvec \le 1$, because $\maxlam({\Mmat}) = 1$
and $\xvec$ is a unit vector. Therefore, using the fact that $\maxlam(\tilde{\Mmat}) = \xvec^\top \tilde{\Mmat} \xvec$ by Courant-Fischer-Weyl\footnote{The Courant-Fischer-Weyl min-max principle states that the maximum eigenvalue $\maxlam(\Mmat)$ of a matrix $\Mmat$
corresponds to $\argmax_{\xvec: \xvec \neq \zerovec} \xvec^\top \Mmat \xvec / (\xvec^\top \xvec)$,
where the corresponding $\xvec$ that gives the maximum is an eigenvector of $\Mmat$. 
Therefore, for this $\xvec$, $\xvec^\top \Mmat \xvec = \maxlam(\Mmat)$
and $\| \xvec \|_2^2 = \xvec^\top \xvec = 1$.}, we get
\begin{align*}
&\maxlam(\tilde{\Mmat}) = \xvec^\top \tilde{\Mmat} \xvec < 1
.
\end{align*}

\noindent
\textbf{Part 2:}
Next we show that further reducing entries, even to zero values,
will not increase the maximum eigenvalue.
This follows simply from the fact that non-negative matrices
are guaranteed to have a non-negative eigenvector $\xvec$ 
that corresponds to the maximum eigenvalue.

To see why, for notational convenience, 
we now let $\Mmat$ be the matrix where entry $i,j$ in $\Ppi$ 
was reduced by $\delta$. 
Let $\tilde{\Mmat}$ further reduce an entry by $\delta$, now potentially to a minimum value of 0, so
that $\tilde{\Mmat}$ is guaranteed to be non-negative (rather than strictly positive).
Using the same argument as above, we obtain that for the non-negative eigenvector 
of $\tilde{\Mmat}$
\begin{align*}
\xvec^\top \tilde{\Mmat} \xvec 
&= \xvec^\top {\Mmat} \xvec - \delta \xvec_i  \xvec_j\\
&\le \maxlam(\Mmat)
\end{align*}
because $\delta \xvec_i  \xvec_j \ge 0$.
Therefore, with further reduction, $\maxlam(\tilde{\Mmat})$
cannot increase and so $\maxlam(\tilde{\Mmat}) \le \maxlam(\Mmat) < 1$.

\noindent
\textbf{Part 3: }
Finally, we can see that for any $\gamma$ and $\Ppi$ 
as given under Assumptions 2 and 3, $\Mmat = \Pgam$
satisfies the above construction. 
\end{proof}

\newcommand{\Pgamlam}{\mathbf{P}_{\pi,\gamma,\lambda}}
\newcommand{\Pgamlamminus}{\mathbf{P}_{\pi,\gamma,1\!-\!\lambda}}

Now we additionally provide definitions for the extension to transition-based discounts.
To do so, we will need to define 
\small
\begin{align*}
\Pgamlam(s,s') &\defeq \sum_{a\in\Actions} \pi(s,a)\Pfcn(s,a,s') \gamma(s,a,s') \lambda(s,a,s')\\
\Pgamlamminus(s,s') &\defeq \sum_{a\in\Actions} \pi(s,a)\Pfcn(s,a,s') \gamma(s,a,s') (1-\lambda(s,a,s'))\\
&= \Pgam - \Pgamlam
\end{align*}
\normalsize
Then we obtain the following generalized definition of $\Pl$ and necessary
properties for convergence within ETD. 
\begin{lemma}~\label{lem_probmatrix}
Under Assumption \asseigs, $\eye - \Pgamlam$
is non-singular and
the matrix 
\begin{align*}
\Pl =  \left(\eye - \Pgamlam \right)^{-1} \Pgamlamminus 
\end{align*}
is non-negative and has rows that sum to no greater than 1, 
\begin{align*}
0 \le \Pl \le 1 \hspace{1.0cm} &\text{ and }  \hspace{1.0cm} \Pl \onevec \le \onevec
.
\end{align*}
\end{lemma}
\begin{proof}

By Lemma \ref{lem_eigs}, we know $\maxlam(\Pgam)< 1$.
Because $0 \le \lambda(s,a,s') \le 1$ for all $(s,a,s')$, this means that
$\maxlam(\Pgamlamminus) < 1$.
Therefore $\eye - \Pgamlam$ is non-singular
and so $\Pl$ is well-defined. 

Notice that $\eye - \Pgamlam$ is a non-singular \textit{M-matrix}, since
the maximum eigenvalue of $\Pgamlam$ is less than one
and entrywise $\Pgamlam \ge 0$.
Therefore, the inverse of $\eye - \Pgamlam$ is positive,
making $\left(\eye - \Pgamlam \right)^{-1} 
\Pgamlamminus $ a positive matrix.
The fact that the matrix has entries that are less than or equal to 1
follows from showing $\Pl \onevec \le \onevec$ below.

To show that the matrix rows always sum to less than 1, we use a simple inductive argument.
Since
\begin{align*}
\Pl =   \sum_{k = 0}^\infty (\Pgamlam)^k \Pgamlamminus 
\end{align*} 
we simply need to show that for every $t$, 
$\sum_{k = 0}^t (\Pgamlam)^k \Pgamlamminus  \onevec \le \onevec$.

\textbf{For the base case, $\mathbf{t = 0}$:} clearly
$$\Pgamlamminus  \onevec \le \onevec $$  

\textbf{For $\mathbf{t > 0}$:}
Assume that 
\begin{align*}
\sum_{k = 0}^t (\Pgamlam)^k \Pgamlamminus  \onevec \le \onevec
.
\end{align*}
Then
\begin{align*}
&\sum_{k = 0}^{t+1} (\Pgamlam)^k \Pgamlamminus  \onevec\\
&=  \Pgamlamminus  \onevec + \left[\sum_{k = 1}^{t+1} (\Pgamlam)^k \Pgamlamminus  \onevec \right]\\
&=  \Pgamlamminus \onevec + \Pgamlam \left[\sum_{k = 0}^{t} (\Pgamlam)^k \Pgamlamminus  \onevec \right]\\
&\le \Pgamlamminus \onevec + \Pgamlam \onevec\\
&= \Pgam\onevec\\
&\le \onevec
\end{align*}
completing the proof.
\end{proof}
 
 \section{Convergence of emphatic algorithms for the RL task formalism}\label{app_proof}

We start with convergence in expectation of ETD for transition-based discounts. 
The results for state-based MDPs should automatically extend to
 transition-based MDPs, due to the equivalence proved in Section \ref{sec_equivalence}. However,
 in an effort to similarly generalize the writing of the theoretical analysis to the more
 general transition-based MDP setting, as we did for algorithms and implementation,
 we explicitly extend the proof for transition-based MDPs.
 
 
%
\begin{theorem} \label{theorem_expectation}
Assume the value function is approximated
using linear function approximation: $\vvec(s) = \xvec(s)^\top \wvec$.
For $\Xmat$ with linearly independent columns (i.e. linearly independent features),
with an interest function $\ivec: \States \rightarrow (0,\infty)$
and $\Mmatpi = \diag(\mvec)$ for $\mvec = (\eye - \Pl)^{-1} (\dvec\hada\ivec)$,
the matrix $\Amat$ is positive definite. 
\begin{align*}
\Amat &\defeq \Xmat^\top \Mmatpi (\eye - \Pgamlam )^\inv (\eye -  \Pgam)  \Xmat\\
&= \Xmat^\top \Mmatpi (\eye - \Pl)  \Xmat
\end{align*}
is positive definite.
\end{theorem} 
\begin{proof}

First, we write an equivalent definition for $\Pl$, 
\begin{align*}
\Pl &= \left(\eye - \Pgamlam \right)^{-1} \Pgamlamminus \\
&=\left(\eye - \Pgamlam \right)^{-1} (\Pgam -  \Pgamlam) \\
&=\left(\eye - \Pgamlam \right)^{-1} (\Pgam -\eye + \eye -  \Pgamlam) \\
    &= \eye -  (\eye-\Pgamlam )^{-1} (\eye-\Pgam) 
    .
\end{align*}
Since $\Xmat$ is a full rank matrix, to prove that 
\begin{align*}
\Amat &=\Xmat^\top \Mmatpi (\eye - \Pl)  \Xmat
\end{align*}
is positive definite, we need to prove that 
$\Mmatpi (\eye - \Pl)$ 
is positive definite.

As in \citep[Theorem 1]{sutton2016anemphatic}, \citep[Proposition C.1]{yu2015onconvergence},
we need to show that for $\Mmatpi (\eye - \Pl)$, 
(a) the diagonal entries are nonnegative, (b) the off-diagonal entries are nonpositive
(c) its row sums are nonnegative and (d) the columns sums are are positive. 
The requirements (a) - (c) follow from Lemma \ref{lem_probmatrix}, 
because $\Mmatpi$ is a non-negative diagonal weighting matrix. 
To show (d), first if $\ivec(s) > 0$ for all $s \in \States$,
the vector of columns sums is
\begin{align*}
\onevec^\top \Mmat (\eye - \Pl) = \mvec^\top (\eye - \Pl) = (\dpi \hada \ivec)^\top
\end{align*}
which always has positive entries.

Otherwise, if $\ivec(s) = 0$ for some $s \in \States$, we can prove
that $\Mmatpi (\eye - \Pl)$ is positive definite using the same argument
as in \citep[Corollary C.1]{yu2015onconvergence}. The proof nicely encapsulates
$\Pl$ generically as a matrix $\Qmat$. We simply have to ensure that the inverse of $\eye - \Pl$ exists
and that $\Pl$ has entries less than or equal to 1, both of which were showed in Lemma \ref{lem_probmatrix}.
The first condition is to have well-defined matrices, and the second
to ensure that $\Qmat$ has a block-diagonal structure.
Therefore, under Assumption 4, we can follow the same proof as \citep[Corollary C.1]{yu2015onconvergence}
to ensure that $\Amat$ is positive definite. 
\end{proof}

For the proofs for ELSTDQ, the main difference is in using action-value
functions. 
We construct the augmented space, with states $\bar{\States} = \States \times \Actions$
and
\small
\begin{align*}
&\Pgamq((s,a), (s,a')) \defeq P(s,a,s')\gamma(s,a,s')\pi(s',a')\\ 
&\Pgamlamq((s,a), (s,a')) \defeq P(s,a,s')\gamma(s,a,s')\lambda(s,a,s')\pi(s',a')
\end{align*}
\normalsize
giving
\begin{align*}
\ivecq((s,a)) &\defeq \ivec(s)\\
\dmuq((s,a)) &\defeq \dmu(s)\mu(s,a)\\
\rlq((s,a)) &\defeq \sum_{s' \in \States} \Pfcn(s,a,s') \Rfcn(s,a,s')\\
\Plq &\defeq (\eye - \Pgamq)^{-1} (\Pgamq - \Pgamlamq)\\
\Mmatq &\defeq \diag\left(\dmuq \hada \ivecq (\eye - \Plq)^{-1}\right)
.
\end{align*}
Then 
\begin{align*}
\Amatq &\defeq \Xmatq^\top \Mmatq (\eye - \Plq) \Xmat\\
\bvecq &\defeq \Xmatq^\top \Mmatq (\eye - \Pgamq) \rlq
.
\end{align*}
The projected Bellman operator is defined as
$\Proj_{\Mmatq} \Tlq \qvec = \rlq + \Plq \qvec$
where ELSTD($\lambda$) converges to the projected
Bellman operator fixed point
$\Proj_{\Mmatq} \Tlq \qvec = \qvec$.

Further, this MDP with the assumptions on the subspace
produced by the state-action features satisfies the conditions
of Theorem \ref{theorem_expectation}, and so $\Amat_q$ is also
positive definite. 

Similarly, the other properties of the Bellman operator
and the weighted norm on $\Plq$ extend,
giving a unique fixed point for the action-value Bellman operator
$\Plq$ and $\| \Plq \|_{\Mmatq} < 1$.
 
 \begin{corollary} \label{cor_elstdq}
Assume the action-value function is approximated
using linear function approximation: $\xvec(s,a)^\top \wvec$.
For $\Xmat$ with linearly independent columns (i.e. linearly independent features),
$\Amatq$
is positive definite. 
\end{corollary}

\section{Issues with transition-based trace without emphatic weighting}

A natural goal is to similarly generalize the contraction properties of
$\Pl$ under the weighting $\dpi$, from constant $\lambda_c$ to transition-based trace. 
To do so, unlike under emphatic weighting, we need to restrict the set of possible trace functions. 
Notice that, because of Assumption A3, for some $s_\lambda < 1$ and  $s_{1-\lambda} < 1$,
for any non-negative $\vvec_+$,
\begin{align*}
&\dpi \Pgamlam \vvec_+ \\
&=
\sum_{s} \sum_a \dpi(s) \Pfcn(s,a,:) \hada \gamma(s,a,:) \hada \lambda(s,a,:) \vvec_+ \\
&\le s_\lambda \sum_{s} \sum_a \dpi(s) \Pfcn(s,a,:) \vvec_+ 
= s_\lambda\dpi \vvec_+
\end{align*}
and similarly 
\begin{align*}
\dpi \Pgamlamminus \vvec_+ &\le
s_{1-\lambda}\dpi \vvec_+
.
\end{align*}
The generalized bound on the $\dpi$ weighted norm is given in the following lemma.
\begin{lemma}
\begin{equation*}
\| \Pl \|_{\Dpi} \le \frac{s_{1-\lambda}}{1-s_\lambda}
.
\end{equation*}
\end{lemma}

 Now, the norm is only a contraction if $s_{1-\lambda} < 1-s_\lambda$. As we have seen,
for constant trace, this inequality holds, since $s_{1-\lambda} = s(1-\lambda)$ and
$s_\lambda = s \lambda$ for some $s < 1$. In general, however, there are
instances where this is not true. We provide such an example below.\footnote{Thanks to an anonymous reviewer
for pointing out this example.}

Consider a 2-state MDP, with uniform probabilities of transitioning and uniform policy, and so $\dpi = [0.5, 0.5]$.
Let $\gamma_c =0.99$ and set $\lambda$ to be $0.9$ when entering state $s_1$ and $0$ when entering state $s_2$.
Then for any $\vvec_+$, 
\begin{align*}
\dpi \Pgamlam \vvec_+ 
&= \gamma_c \dpi \Ppi \left[\begin{array}{cc}
0.9 & 0\\
0 & 0
\end{array} \right] \vvec_+\\
&= \gamma_c \dpi \left[\begin{array}{cc}
0.9 & 0\\
0 & 0
\end{array} \right] \vvec_+\\
&= \gamma_c 0.9 \dpi \left[\begin{array}{cc}
1.0 & 0\\
0 & 0
\end{array} \right] \vvec_+\\
&\le \gamma_c 0.9 \dpi \vvec_+
\end{align*}
 where for $\vvec_+ = [v \ \ 0]^\top$ for any $v \ge 0$, this bound is tight. 
 Similarly,
 \begin{align*}
\dpi \Pgamlamminus \vvec_+ 
&= \gamma_c \dpi \Ppi \left[\begin{array}{cc}
0.1 & 0\\
0 & 1.0
\end{array} \right] \vvec_+\\
&\le \gamma_c \dpi \vvec_+
.
\end{align*}
 where for $\vvec_+ = [0\ \ v]^\top$ for any $v \ge 0$, this bound is tight. 

  Therefore, 
  $s_\lambda = 0.9 \gamma_c$ and $s_{1-\lambda} = \gamma_c$,
  and so we get $1-s_\lambda = 0.9 \gamma_c < s_{1-\lambda} = \gamma_c$,
  which makes the upper bound in the above lemma $1.\bar{1}$.
 Computing $\Pl$,
 \begin{align*}
\Pl = \left[\begin{array}{cc}
0.0893 & 0.8927\\
0.0893 & 0.8927
\end{array} \right]
.
\end{align*}
 we can see that this is not a contraction. 
\end{document}